\documentclass[twoside]{article}

\usepackage[accepted]{aistats2023}
\usepackage[round]{natbib}

\usepackage{hyperref}       % hyperlinks
\usepackage{url}            % simple URL typesetting
\usepackage{booktabs}       % professional-quality tables
\usepackage{amsfonts}       % blackboard math symbols
\usepackage{nicefrac}       % compact symbols for 1/2, etc.
\usepackage{xcolor}         % colors
\usepackage{wrapfig}
\usepackage{algorithm}
\usepackage{algorithmic}
\usepackage{caption}
\usepackage{multicol}
\usepackage{subcaption}
\usepackage[export]{adjustbox}
\usepackage{enumitem}
\usepackage{multirow}
\usepackage{arydshln}
\usepackage{threeparttable}
\usepackage{multicol}
\usepackage{xcolor,colortbl}

\usepackage{cellspace}
\usepackage{pifont}% http://ctan.org/pkg/pifont
\newcommand{\cmark}{\textcolor{green!50!black}{\ding{51}}}%
\newcommand{\xmark}{\textcolor{red}{\ding{55}}}%

\usepackage{mymacros}
\newcommand{\horizon}{T}
\usetikzlibrary{automata,arrows,positioning,calc}

\definecolor{mygreen}{rgb}{0.0, 0.5, 0.0}

\allowdisplaybreaks[4]

\begin{document}
\setlength{\abovedisplayskip}{3pt}
\setlength{\belowdisplayskip}{3pt}



\twocolumn[

\aistatstitle{A Tale of Sampling and Estimation in Discounted Reinforcement Learning}

\aistatsauthor{ Alberto Maria Metelli \And Mirco Mutti \And  Marcello Restelli }

\aistatsaddress{ Politecnico di Milano \And  Politecnico di Milano \And Politecnico di Milano } ]

\begin{abstract}
The most relevant problems in discounted reinforcement learning involve estimating the mean of a function under the stationary distribution of a Markov reward process, such as the expected return in \emph{policy evaluation}, or the policy gradient in \emph{policy optimization}.
In practice, these estimates are produced through a finite-horizon episodic sampling, which neglects the mixing properties of the Markov process.
It is mostly unclear how this mismatch between the practical and the ideal setting affects the estimation, and the literature lacks a formal study on the pitfalls of episodic sampling, and how to do it optimally.
In this paper, we present a minimax lower bound on the discounted mean estimation problem that explicitly connects the estimation error with the mixing properties of the Markov process and the discount factor.
Then, we provide a statistical analysis on a set of notable estimators and the corresponding sampling procedures, which includes the finite-horizon estimators often used in practice.
Crucially, we show that estimating the mean by directly sampling from the discounted kernel of the Markov process brings compelling statistical properties \wrt the alternative estimators, as it matches the lower bound without requiring a careful tuning of the episode horizon.
\end{abstract}

\section{INTRODUCTION}
The discounted formulation of the Markov Decision Process~\citep[MDP,][]{puterman2014markov}, initially studied in~\citep{blackwell1962discrete,bellman1966dynamic}, established itself as one of the most popular models for Reinforcement Learning~\citep[RL,][]{sutton2018reinforcement} due to its favorable theoretical tractability and its link with temporal difference learning~\citep{sutton1988learning}, a key ingredient behind several successful algorithms~\citep[\eg][]{watkins1992q, mnih2015human, lillicrap2016continuous, silver2016mastering}.
On a technical level, discounted RL problems are based on the estimation of \quotes{exponentially discounted} quantities over an infinite horizon. Specifically, \emph{policy evaluation} requires estimating the $\gamma$-discounted value function $V^\mu_\gamma$ of a policy $\mu$:
\begin{equation}
	V^\mu_\gamma (s) = \E_{\mu} \left[ \sum_{t = 0}^{+\infty} \gamma^t R(s_t, a_t) \ \bigg| \ s_0 = s \right],
	\label{eq:value_function}
\end{equation}
whereas \emph{policy optimization} involves the estimation of the policy gradient~\citep{sutton1999policy}:
\begin{equation}
	\nabla_{\mu} V^\mu_\gamma = \E_{\mu} \left[ \sum_{t = 0}^{+\infty} \gamma^t \nabla_\mu \log \mu (a_t | s_t) Q_\gamma^\mu (s_t, a_t)  \right].
	\label{eq:policy_gradient}
\end{equation}
One can equivalently write those quantities as expectations over the state-action space $\E_{(s, a) \sim \pi^\mu_\gamma} [ f(s, a) ]$, where $\pi^\mu_\gamma(s, a) \coloneqq (1-\gamma) \E_{\mu} [\sum_{t = 0}^{+\infty} \gamma^t \indic \{s_t = s, a_t = a\}]$ is the \emph{$\gamma$-discounted state-action distribution} induced by policy $\mu$. The latter can be seen as the stationary distribution of a suitably defined Markov Chain~\citep[MC,][]{levin2017markov} obtained from the original MDP by fixing the policy $\mu$ and considering, at any step, a reset probability $1 - \gamma$ of returning to the initial state. This means that discounted RL is rooted in the mean estimation of a function in an MC. Nonetheless, the latter technical problem has received little attention in the discounted RL literature, which has mostly focused on the pitfalls of common practices~\citep[\eg][]{thomas2014bias, lehnert2018value, nota2020policy, tang2021taylor, zhang2022deeper} and their impact on the learning problem~\citep{jiang2016structural, van2019using, amit2020discount, guo2022theoretical}. Instead, how to optimally \emph{collect samples} from the MC and how to appropriately perform the \emph{mean estimation} remain mostly obscure.

This paper formally studies the $\gamma$-discounted mean estimation in MCs. First, we provide a general formulation of the problem by defining an estimation algorithm as a pairing of a \emph{reset policy}, which is used to make decisions on whether to reset the chain (\ie re-start from the initial state) at a given step, and an actual \emph{estimator}, from which the estimate is computed on the collected samples. For this notion of estimation algorithm, we introduce a PAC requirement that guarantees a small estimation error with high probability when enough samples are collected. Most importantly, we derive a \emph{lower bound} on the number of samples required by any estimation algorithm to meet the proposed PAC requirement, which relates the sample complexity to the discount factor $\gamma$ and the mixing properties of the chain.

Having established the statistical barriers of the problem, we shift our focus toward the properties of practical estimation algorithms. The most common practice in discounted RL is to compute the quantities in Equations~(\ref{eq:value_function},~\ref{eq:policy_gradient}) through a \emph{finite-horizon} algorithm, \ie that resets the chain every $\horizon$ steps. This approach is known to suffer from a meaningful \emph{bias}~\citep{thomas2014bias}, which can be only partially mitigated with a careful choice of $\horizon$ and correcting factors (though they are seldom used in practice). Alternatively, one can design an \emph{unbiased} estimation algorithm that, at every step, rejects the collected sample with probability $\gamma$, otherwise it accepts the sample and resets the chain. However, this \emph{one-sample} estimator has been mostly used as a theoretical tool~\cite[\eg][]{thomas2014bias, metelli2021safe}, since it wastes a large portion of the samples and, thus, suffers from a large variance. Another option is to reset the chain like the \emph{one-sample} estimator, but to compute the estimate over all the collected samples, like the \emph{finite-horizon} estimators. The latter \emph{all-samples} approach introduces some bias by bringing dependent samples, but it mitigates the variance through a greater effective sample size. For the mentioned estimation algorithms, we study their computational properties, derive concentration inequalities, and certify the \emph{all-samples} approach, which practitioners almost neglect, is the one with the best statistical profile, as it results nearly minimax optimal. 
%Furthermore, we compare the presented estimators in terms of parallelizability and computational complexity.

\textbf{Original Contributions}~~~In summary, we contribute:
\begin{itemize}[noitemsep,topsep=-2pt,parsep=0pt,partopsep=0pt,leftmargin=*]
	\item A formal definition of the problem of $\gamma$-discounted mean estimation in Markov chains and its corresponding PAC requirement (Section~\ref{sec:3});
	\item The first minimax lower bound of order $\widetilde{\Omega}( 1 / \sqrt{ N (1 - \beta\gamma) } )$ on the error of $\gamma$-discounted mean estimation in MCs, where $N$ is the number of collected samples and $1 - \beta$ is the \emph{absolute spectral gap}~\citep{levin2017markov} of the chain (Section~\ref{sec:minimax_lower_bound});
	\item The analysis of the statistical properties of a family of estimation algorithms that includes the \emph{finite-horizon}, \emph{one-sample}, \emph{all-samples} types (Section~\ref{sec:estimators}), in which the \emph{all-samples} approach results in the best statistical profile;
	\item An empirical evaluation of the mentioned estimation algorithms over simple yet illustrative problems, which uphold the compelling statistical properties of the \emph{all-samples} estimator (Section~\ref{sec:validation}).
\end{itemize}

Finally, this paper aims to shed light on the statistical barriers of $\gamma$-discounted mean estimation in MCs, which stands as the technical bedrock of the discounted RL formulation. As a by-product of this theoretical analysis, the \emph{all-samples} estimation approach emerges as an interesting opportunity for the development of novel practical algorithms for discounted RL supported by compelling statistical properties.

\setlength{\textfloatsep}{8pt}

\section{PRELIMINARIES}
In this section, we introduce the necessary background that will be employed in the subsequent sections of the paper. 

\textbf{Notation}~~Let $\Xs$ be a set and $\mathfrak{F}$ a $\sigma$-algebra on $\Xs$. We denote with $\PM{\Xs}$ the set of probability measures over $(\Xs, \mathfrak{F})$. We denote with $\MF{\Xs}$ the set of $\mathfrak{F}$-measurable real-valued functions. Let $\nu \in \PM{\Xs}$ and $f \in \MF{\Xs}$, with little abuse of notation, we denote with $\nu: \MF{\Xs} \rightarrow \Reals$ the expectation operator $\nu f = \int_{\Xs} f(x) \nu(\de x)$. For $p \in [1,\infty)$, we define the $L_p(\pi)$-norm as $\left\|f \right\|_{\pi,p}^p = \int_{\Xs} |f(x)|^p \pi(\de x)$. Let $T :\MF{\Xs} \rightarrow \MF{\Xs} $ be a linear operator, we define the operator norm as $\left\| T  \right\|_{\pi, p \rightarrow q} = \sup_{\left\| f \right\|_{\pi,p} \le 1 }  \| T f \|_{\pi,q}$, for $p,q \in[1,\infty)$. Let $\nu,\mu \in \PM{\Xs}$, the \emph{chi-square divergence} is defined as $\chi_2(\nu\|\mu) = \left\| (\mathrm{d} \nu / \mathrm{d} \mu -1)^2 \right\|_{\mu}^2$. For $a,b \in \Nat$ with $a \le b$ we employ the notation $\ldsquare a,b \rdsquare = \{a,\dots,b\}$.
%
%For every $p \in [1,\infty)$, we define $\mathscr{L}_p(\Xs, \pi) = \{ f\in \MF{\Xs} : \left\|f \right\|_{\pi,p}^p = \int_{\Xs} |f(x)|^p \pi(\de x) < +\infty \} $ the set of measurable functions with bounded $L_p(\pi)$-norm and for $p = \infty$ we define $\mathscr{L}_\infty(\Xs, \pi) = \left\{  f\in \MF{\Xs} : \|f \|_{\pi,\infty} = \esssup_{x \in \Xs, \pi} |f(x)|<+\infty\right\}$ the set of measurable functions with bounded essential supremum. Let $T :\mathscr{L}_p(\Xs, \pi) \rightarrow \mathscr{L}_q(\Xs, \pi) $ be a linear operator, we define the operator norm as $\left\| T  \right\|_{\pi, p \rightarrow q} = \sup_{\left\| f \right\|_{\pi,p} \le 1 }  \| T f \|_{\pi,q}$. The \emph{total variation} divergence between two probability measures $\nu,\mu \in \PM{\Xs}$ is defined as: $\left\| \nu  - \mu \right\|_{\text{TV}} = \sup_{\mathcal{B} \in \mathfrak{F}} \nu(\mathcal{B}) - \mu(\mathcal{B}) = \sup_{f \in \mathscr{L}_{\text{B}}} (\nu - \mu) f$.
%
%For $a,b \in \Nat$ with $a \le b$ we denote $\ldsquare a,b \rdsquare = \{a,\dots,b\}$.

\textbf{Markov Chains}~~
A Markov kernel is an $\mathfrak{F}$-measurable function $P: \Xs \rightarrow \PM{\Xs}$ mapping every state $x \in \Xs$ to a probability measure $P(\cdot|x) \in \PM{\Xs}$. We denote with $\PM{\Xs,\Xs}$ the set of Markov kernels over $(\Xs,\mathfrak{F})$. With little abuse of notation, we denote with the same symbol the operator $P: \MF{\Xs} \rightarrow \MF{\Xs}$ defined as $ (P f)(x) = \int_{\Xs}  f(y) P(\de y|x)$ for $x \in \Xs$. A probability measure $\pi \in \PM{\Xs}$ is \emph{invariant} \wrt $P$ if $\pi = \pi P$. Let $\Pi = \mathbf{1} \pi$, we define the \emph{absolute $L_2$-spectral gap}~\citep{levin2017markov} as $1-\beta$, where $\beta = \left\| P - \Pi \right\|_{\pi,2\rightarrow 2}$.

\textbf{Discounted Sampling}~~Let $\gamma \in [0,1]$ be a discount factor, $P \in \PM{\Xs,\Xs}$ be a Markov kernel, and $\nu \in \PM{\Xs}$ be an initial-state distribution, the \emph{$\gamma$-discounted stationary distribution} $\pi_\gamma \in \PM{\Xs}$ is defined in several equivalent forms:
\begin{equation}\label{eq:gammaDiscDist}
\begin{aligned}
	\pi_\gamma 
	&= (1-\gamma) \sum_{t \in \Nat} \gamma^t \nu P^t \\ 
	&= (1-\gamma) \nu\left( I - \gamma P\right)^{-1} = (1-\gamma) \nu + \gamma \pi_\gamma P.
\end{aligned}
\end{equation}
$\pi_\gamma$ represents normalized expected count of the times each state is visited, where a visit at time $t \in \Nat$ counts $\gamma^t$.
If $\gamma < 1$, $\pi_\gamma$ is guaranteed to exist. When $\gamma=1$, we denote with $\pi= \lim_{\gamma\rightarrow 1} \pi_\gamma$ the \emph{stationary distribution} of $P$, if it exists. It is well-known that $\pi_\gamma$ is also the stationary distribution of the MC with kernel $P_\gamma = (1-\gamma) \mathbf{1} \nu + \gamma P$.

\section{$\gamma$-DISCOUNTED MEAN ESTIMATION}\label{sec:3}
In this section, we formally define the problem of $\gamma$-discounted mean estimation in Markov chains. Then, we introduce a general framework for characterizing a broad class of estimators (Section~\ref{sec:resetAlg}), and we formally define the PAC requirement to asses their quality (Section~\ref{sec:PAC}).

\begin{figure*}[t]
\centering
\includegraphics[width=.85\textwidth]{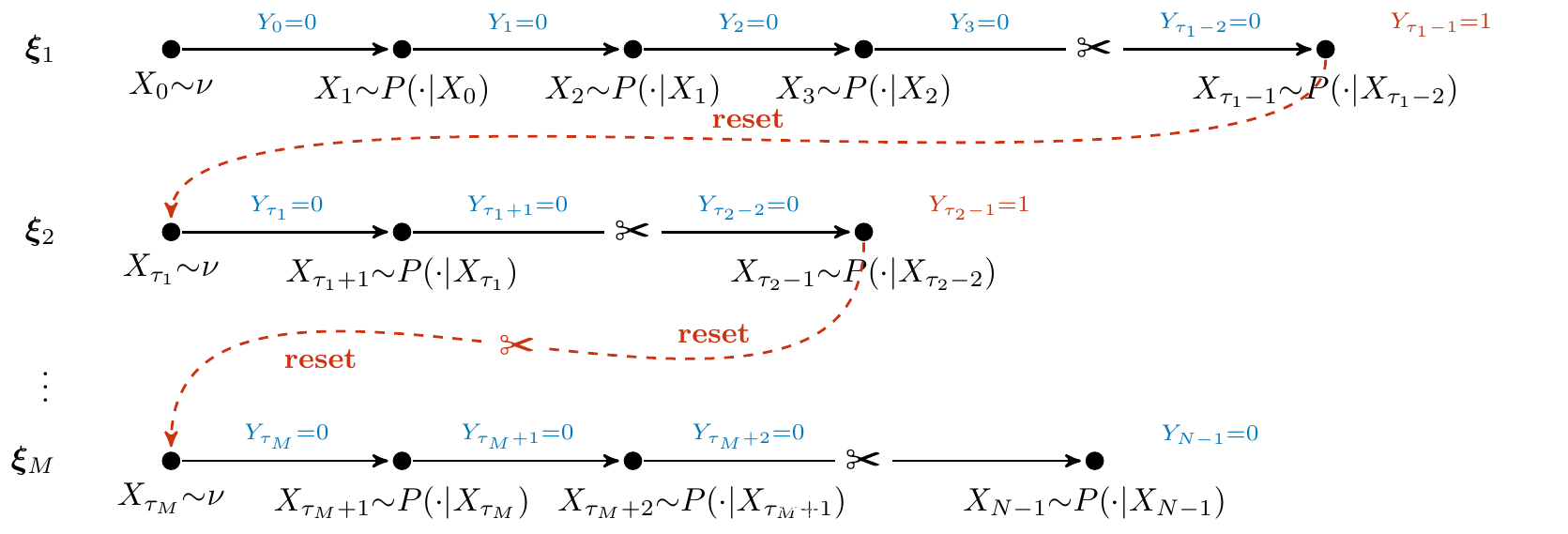}
\caption{Graphical representation of the sampling process in a Markov chain with a reset policy.}\label{fig:repr}
\end{figure*}

Let $\nu \in \PM{\Xs}$ be an initial-state distribution and $P \in \PM{\Xs,\Xs}$ be a Markov kernel. Given a discount factor $\gamma \in [0,1]$ and a measurable function $f \in \MF{\Xs}$, our goal consists in estimating:
\begin{align}
\pi_\gamma f \coloneqq \E_{X \sim \pi_\gamma}[f(X)]  = \int_{\Xs} f(x) \pi_\gamma(\de x),
\end{align} 
\ie the expectation of $f$ under the $\gamma$-discounted stationary distribution $\pi_\gamma$ induced by $\nu$ and $P$, as defined in Equation~\eqref{eq:gammaDiscDist}. Furthermore, we introduce the quantity:\footnote{$\pi_\gamma : \MF{\Xs} \rightarrow \Reals$ and $\sigma_\gamma^2 : \MF{\Xs} \rightarrow \Reals[0]$ act as operators.}
\begin{align*}
	\sigma_\gamma^2 f \coloneqq \Var_{X \sim \pi_\gamma}[f(X)] = \int_{\Xs} (f(x) - \pi_\gamma f)^2 \pi_\gamma(\de x),
	%&=  = \pi_\gamma \left(f - \Pi_\gamma f \right)^2 \\ 
	%&= \int_{\Xs} (f(x) - \pi_\gamma f)^2 \pi_\gamma(\de x).
\end{align*}
\ie the variance of $f$ under distribution $\pi_\gamma$. When $\gamma = 1$, we denote with $\pi f$ and $\sigma^2 f$ the expectation and variance of $f$ under the stationary distribution, when they exist.

\begin{algorithm}[t]
	\caption{Markov chain sampling with reset policy.}\label{alg:sampling}
	{\small
	\textbf{Input}: Markov kernel $P$, initial-state distribution $ \nu$, discount factor $\gamma$, reset policy $\vrho = (\rho_t)_{t\in \Nat}$, number of samples $N$\\
	\textbf{Output}: dataset ${H}_{N}$
	\begin{algorithmic}
		\STATE ${H}_0 = ()$, $X_0 \sim \nu$
		\FOR{$t \in \dsb{0,N-1}$}
			\STATE $Y_t \sim \rho_t(\cdot|\mathcal{H}_t,X_t)$
			\STATE ${H}_{t+1} = {H}_t \oplus ((X_t,Y_t))$  $\qquad\quad \oplus$ denotes concatenation
			\IF{$Y_t = 0$}
				\STATE $X_{t+1} \sim P(\cdot|X_t)$
			\ELSE
				\STATE $X_{t+1} \sim \nu$
			\ENDIF
		\ENDFOR
		\RETURN ${H}_N$
	\end{algorithmic}}
\end{algorithm}

\subsection{Reset-based Estimation Algorithms}\label{sec:resetAlg}
In order to perform a reliable estimation of $\pi_\gamma f $, it is advisable to have the possibility of \quotes{resetting} the chain, \ie to interrupt the natural evolution of the chain based on the Markov kernel $P$, and restart the simulation from a state sampled from the initial-state distribution $\nu$. 
%As we shall see in Section~\ref{sec:reset_necessary}, without reset the estimation problem becomes infeasible. 
For these reasons, we introduce the notion of \emph{reset policy}, \ie a device that decides whether to reset the chain, based on time, the current history of observed states and reset decisions.

\begin{defi}[Reset Policy]
A \emph{reset policy} is a sequence $\vrho = (\rho_t)_{t \in \Nat}$ of functions $\rho_t : \Hs_{t} \times \Xs \rightarrow \PM{\{0,1\}}$ mapping for every $t \in \Nat$ a  history of past states and resets $H_{t} = (X_0,Y_0,\dots,X_{t-1},Y_{t-1}) \in \Hs_{t} = (\Xs \times \{0,1\})^{t}$ and the current state $X_t \in \Xs$, to a probability measure  $\rho_t(\cdot|H_{t},X_t) \in \PM{\{0,1\}}$.
\end{defi}

Thus, at every time instant $t \in \Nat$, based on $H_{t}\in \Hs_t$ and $X_t \in \Xs$, we sample the reset decision $Y_t \in \{0,1\}$ from the current reset policy $\rho_t(\cdot|H_{t},X_t)$. If $Y_t =1$, then we reset, \ie the next state $X_{t+1} $ is sampled from the initial state distribution $\nu$, whereas if $Y_t =0$ the chain evolution proceeds and $X_{t+1}$ is sampled from the Markov kernel $P(\cdot|X_t)$. The resulting sampling algorithm is reported in Algorithm~\ref{alg:sampling}.

\textbf{Resettable and Resetted Processes}~~Given a reset policy  $\vrho = (\rho_t)_{t \in \Nat}$, we can represent the distribution of the next state with the \emph{resettable process}, defined through the resettable Markov kernel $P_{\nu}: \Xs \times \{0,1\} \rightarrow \PM{\Xs}$, defined for every $(X,Y) \in \Xs \times \{0,1\}$ and measurable set $\mathcal{B} \in \mathfrak{F}$ as:
\begin{align}
	P_{\nu}(\mathcal{B}|X,Y) = Y \cdot\nu(\mathcal{B}) + (1-Y) \cdot P(\mathcal{B}|X).
\end{align}
Suppose we run the process for $N \in \Nat$ steps, the product measure generating the history $H_N$ is given by $P^N_{\nu,\vrho} = \nu \otimes \rho_0 \otimes ( \bigotimes_{t=1}^{N-1}  P_{\nu} \otimes \rho_t ) $. 
The sequence of states resulting from applying a reset policy $\vrho$ is a non-stationary non-Markovian process, called \emph{resetted process}, whose kernel $P_{\nu,\rho_t,t}: \Hs_{t} \times \Xs \rightarrow \PM{\Xs}$ is defined for $t \in \Nat$, history $H \in \Hs_{t}$, state $X \in \Xs$, and measurable set $\mathcal{B} \in \mathfrak{F}$ as:\footnote{If $\rho_t$ is stationary and/or Markovian then  $P_{\nu,\rho_t,t}$ is stationary and/or Markovian.}
\begin{equation}
\begin{aligned}
	P_{\nu,\rho_t,t}(\mathcal{B}|X,H) &= \rho_t(\{1\}|H,X) \cdot \nu(\mathcal{B})  \\
	&\quad + \rho_t(\{0\}|H,X) \cdot P(\mathcal{B}|X).
\end{aligned}
\end{equation}

\textbf{Trajectories and Horizons}~~
%, given a budget of $N \in \Nat$ samples. 
%
%
% we assume to have the possibility of generating samples from the environment and reset the environment whenever we want. Given a discount factor $\gamma \in [0,1]$ and a bounded measurable function $f$, our goal consists in estimating the expectation $\pi_\gamma f = \int_{\Xs} f(x) \pi_\gamma(\de x)$ under the $\gamma$-discounted stationary distribution $\pi_\gamma$, given a budget of $N \in \Nat$ samples. 
%Any estimator can be defined by means of a reset policy $\rho_t$ that, based on the history of observations and time, decides whether to reset the chain. A \emph{reset policy} is defined for every $t \in \Nat$ as $\rho_t : \mathcal{H}_{t-1} \times \Xs\rightarrow \PM{\{0,1\}}$ where $\mathcal{H}_{t-1} = (\Xs \times \{0,1\})^{t-1} $. The result is a sequence $\mathcal{H}_N$. 
We define a \emph{trajectory} as the sequence of states observed between two consecutive resets. The number of trajectories $M$ can be computed in terms of the resets, \ie $M = 1 + \sum_{t=0}^{N-1} Y_i$. We introduce the time instants $\tau_i$ in which a reset is performed as:
{\thinmuskip=1mu
\medmuskip=1mu
\thickmuskip=1mu
\begin{align}
\tau_i = \begin{cases}
	0 & \text{if } i = 1 \\
	 1 + \min\{ t \in \Nat \,:\, t \ge \tau_{i-1}  \wedge Y_t = 1\} & \text{if } i \in \dsb{2,M} \\
	 N & \text{if } i = M+1
\end{cases}.
\end{align}
}%
%\begin{align}
%	& \tau_1 = 0, \nonumber \\
%	& \tau_i = 1 + \min\{ t \,:\, t \ge \tau_{i-1}  \wedge Y_t = 1\} \;\; \forall i \in \{2,\dots,M\}, \nonumber \\ 
%	& \tau_{M+1} = N.
%\end{align}
Therefore, for every $i \in \dsb{M}$, a trajectory is given by $\vxi_i = (X_{\tau_{i}}, \dots, X_{\tau_{i+1}-1})$, whose horizon is computed as $\horizon_i = \tau_{i+1} - \tau_{i}$. A graphical representation of the resulting sampling process is provided in Figure~\ref{fig:repr}.

%
%\paragraph{Examples of Reset}
%We consider two reset policies:
%
%\begin{minipage}{.45\textwidth}
%\fbox{\parbox{\textwidth}{
%\centering
%\textbf{Fixed-Horizon Reset (FHR)}
%\begin{align}
%\rho_t^{\text{FHR}}(\cdot|H_t,X_t) = \delta_{\mathbf{1}\{t \bmod T = 0\}}
%\end{align}
%
%\vspace{-.2cm}
%}}
%\end{minipage}%
%\hfill
%\begin{minipage}{.45\textwidth}
%\fbox{\parbox{\textwidth}{
%\centering
%\textbf{Adaptive-Horizon Reset (AHR)}
%\begin{align}
%\rho_t^{\text{AHR}}(\cdot|H_t,X_t) = \mathrm{Ber}(1-\gamma)
%\end{align}
%
%\vspace{-.2cm}
%}}
%\end{minipage}%

%The reset policy $\vrho = (\rho_t)_{t \in \Nat}$ induces a chain $P^\nu_t(\cdot|X_t,Y_t) = Y_t \nu(\cdot) + (1-Y_t)P(\cdot|X_t)$, which is stationary  

\subsection{Estimators and PAC Requirement}\label{sec:PAC}
In addition to the reset policy $\vrho$, to actually define an estimation algorithm, we need an \emph{estimator}, \ie a function $\widehat{\eta} : \Hs_N \times \MF{\Xs} \rightarrow \Reals$ that maps a history of observations $H_N \in \Hs_N$ and a measurable function $f \in \MF{\Xs}$ to a real number $\widehat{\eta}(H_N, f) \in \Reals$. Thus, an \emph{estimation algorithm} is a pair $\mathfrak{A} = (\vrho, \widehat{\eta})$. We now introduce the PAC requirement to assess the quality of an estimation algorithm $\mathfrak{A} $.

\begin{defi}[$(\epsilon,\delta,N)$-PAC] Let $\gamma \in [0,1]$ be a discount factor, let $P \in \PM{\Xs, \Xs}$ be a Markov kernel, let $\nu \in \PM{\Xs}$ be an initial-state distribution, and let $f \in \MF{\Xs}$ be a measurable function. An estimation algorithm $\mathfrak{A} = (\textcolor{black}{\bm{\rho}},\textcolor{black}{\widehat{\eta}})$ for the $\gamma$-discounted mean $\pi_\gamma f$, is $(\epsilon,\delta,N)$-PAC if with probability at least $1 - \delta$ it holds that: 
\begin{equation*}
	\left| \textcolor{black}{\widehat{\eta}}(H_{N},f) - \pi_\gamma f \right| < \epsilon,
\end{equation*}
where $H_N \sim P^N_{\nu,\textcolor{black}{\vrho}}$ is collected with the reset policy $\textcolor{black}{\vrho}$.
\end{defi}
%\begin{defi}[$(\epsilon,\delta,N)$-PAC] Let $\nu \in \PM{\Xs}$ be the initial state distribution, let $P \in \PM{\Xs, \Xs}$ be a Markov kernel, let $f \in \MF{\Xs}$ be a measurable function, and let $\gamma \in [0,1]$ be a discount factor. An estimator $\mathfrak{E} = (\textcolor{black}{\bm{\rho}},\textcolor{black}{\widehat{\eta}})$ for the $\gamma$-discounted expectation of $f$, \ie $\pi_\gamma f$, is $(\epsilon,\delta,N)$-PAC if: 
%\begin{align*}
%	\Prob_{H_N \sim P^N_{\nu,\textcolor{black}{\vrho}} }\left( \left| \textcolor{black}{\widehat{\eta}}(H_{N}) - \pi_\gamma f \right| > \epsilon \right) \le \delta.
%\end{align*}
%\end{defi}

In the next sections, we first dive into the study of the intrinsic complexity of estimating $\pi_\gamma f$ (Section~\ref{sec:minimax_lower_bound}) and, then, we present a handful of practical estimators along with their computational and statistical properties (Section~\ref{sec:estimators}).

\section{MINIMAX LOWER BOUND FOR $\gamma$--DISCOUNTED MEAN ESTIMATION}\label{sec:minimax_lower_bound}
In this section, we prove the first minimax lower bound for the problem of $\gamma$--discounted mean estimation in MCs. We first state a lower bound for a general estimation algorithm $\mathfrak{A} = (\vrho, \widehat{\eta})$. Then, we report a brief sketch of the proof, which includes how to construct the hard instance, while a complete derivation can be found in Appendix~\ref{apx:proofs_minimax}.
%Then, we show that the possibility of resetting the chain is actually necessary to achieve a reliable estimation (Section~\ref{sec:reset_necessary}).
%\subsection{Minimax Lower Bound}\label{sec:minimax}

\begin{restatable}[Minimax Lower Bound]{thr}{minimax}\label{thr:minimax}
For every discount factor $\gamma \in [0,1]$, sufficiently small confidence $\delta$,\footnote{The explicit regime for $\delta$ is reported in the proof sketch.} number of interactions $N \in \Nat$, and $(\epsilon,\delta,N)$-PAC estimation algorithm $\mathfrak{A} = (\textcolor{black}{\bm{\rho}},\textcolor{black}{\widehat{\eta}})$, there exists a class of Markov kernels $P \in \PM{\Xs,\Xs}$ with absolute spectral gap $1-\beta \in (0,1]$, initial-state distributions $\nu\in \PM{\Xs}$, measurable function $f \in \MF{\Xs}$ such that with probability at least $\delta$ it holds that:
\begin{equation*}
%\inf_{\vrho, \widehat{\eta}} \sup_{\substack{P,\nu,f \\ \text{with spectral gap $1-\beta$}}} 
\left| \textcolor{black}{\widehat{\eta}}(H_{N},f) - \pi_\gamma f \right| \geq \sqrt{\frac{\sigma^2_\gamma f \cdot \log \frac{1}{2\delta}}{N (1 - \beta\gamma)}},
\end{equation*}
where $H_N \sim P^N_{\nu,\textcolor{black}{\vrho}}$ is collected with the reset policy $\textcolor{black}{\vrho}$.
\end{restatable}
%\begin{restatable}[Minimax Lower Bound]{thr}{minimax}
%For every discount factor $\gamma \in [0,1]$, there exists a class of MCs $P$ with absolute spectral gap $1-\beta \in (0,1]$, initial state distributions $\nu$, measurable function $f$ such that for every accuracy $\epsilon \in [0,1]$ and number of interactions $N \in \Nat$ and every $(\epsilon,\delta,N)$-PAC estimator $\mathfrak{E} = (\textcolor{black}{\bm{\rho}},\textcolor{black}{\widehat{\eta}})$ it holds that:
%\begin{align*}
%\inf_{\vrho, \widehat{\eta}} \sup_{\substack{P,\nu,f \\ \text{with spectral gap $1-\beta$}}} \Prob_{H_N \sim P^N_{\nu,\textcolor{black}{\vrho}} }\left( \left| \textcolor{black}{\widehat{\eta}}(H_{N},f) - \pi_\gamma f \right| \geq \epsilon \right) \\
%\ge \begin{cases}
%\frac{1}{2} \exp\left( - \frac{\epsilon^2 N (1-\beta\gamma)}{ \sigma^2_\gamma f} \right) & \text{if } \epsilon \in \left[0, \frac{1-\beta}{1-\beta\gamma}\right] \\
% \frac{1}{2} \exp\left( - \frac{\epsilon^2 N}{ \sigma^2_\gamma f} \right) & \text{otherwise}
%\end{cases}.
%\end{align*}
%\end{restatable}
\begin{proofsketch}
	The proof is based on the MC construction:
	\vspace{-.4cm}
	\begin{center}
	\begin{tikzpicture}[->, >=stealth', auto, semithick, node distance=3cm, scale=.8]
		\tikzstyle{every state}=[fill=white,draw=black,thick,text=black,scale=1]
		\node[state]    (A)             {$A$};
		\node[state]    (B)[right of=A]   {$B$};
		\path
		(A) edge[loop left]     node{$p+\beta$}         (A)
    		edge[bend left,above]      node{$1-p-\beta$}      (B)
		(B) edge[loop right]    node{$1-p$}     (B)
    		edge[bend left,below]     node{$p$}         (A);
	\end{tikzpicture}
	\end{center}\vspace{-.5cm}
	having two states $\Xs = \{A, B\}$, kernel $P$ parametrized via $\beta \in [0, 1)$ and $p \in (0, 1 - \beta)$, initial state distribution $\nu = (q, 1 - q)$ parametrized via $q \in (0, 1)$. By computing the invariant measure $\pi$ for the kernel $P$, it is easy to verify that the MC has spectral gap $1 - \beta$ for every value of $p$.
	
	We consider a pair of functions $f_1, f_{-1}$ defined as $f_1 (B) = f_{-1} (B) = 0$, $f_1 (A) = 1$, $f_{-1} (A) = -1$. Crucially, any estimator cannot distinguish the two functions if the state $A$ is never visited. With this intuition, we can lower bound the probability of making an error $\epsilon \in [0, 1]$ through the probability of visiting $A$. For $p$ and $q$ such that $\pi_\gamma f_1 = \epsilon$ (consequently $\pi_\gamma f_{-1} = -\epsilon$), we can derive:
	\begin{align*}
	\sup_{\substack{P,\nu,f \\ \text{with spectral gap $1-\beta$}}} & \Prob_{H_N \sim P^N_{\nu,\textcolor{black}{\vrho}} }\left( \left| \textcolor{black}{\widehat{\eta}}(H_{N},f) - \pi_\gamma f \right| \geq \epsilon \right) \\
	& \ge \frac{1}{2} \Prob_{H_N \sim P^N_{\nu,\textcolor{black}{\vrho}} }  \left(  \textcolor{black}{\widehat{\eta}}(H_{N},f_{-1}) =  \textcolor{black}{\widehat{\eta}}(H_{N},f_{1})\right) \\
	& \ge \frac{1}{2}  \min \{1-q,1-p\}^{N}.
	\end{align*}
	Then, we optimize the values of $p$ and $q$ to make the bound tight, and with some algebraic manipulations we get:
	\begin{equation*}
	\frac{1}{2}  \min \{1-q,1-p\}^{N} \ge \exp \left( - \frac{\epsilon^2 N (1 - \beta \gamma)}{\sigma^2_\gamma f} \right),
	\end{equation*}
	in the regime $\epsilon \in \big[ 0, \frac{1 - \beta}{1 - \beta\gamma} \big]$. The statement follows by reformulating the lower bound in terms of deviation $\epsilon$ for a small enough $\delta \in \big( 0, \frac{1}{2} \exp \big(- \frac{N (1 - \beta)^2}{ \sigma^2_\gamma f (1 - \beta \gamma)} \big) \big)$.
\end{proofsketch}

The presented minimax lower bound establishes an instance-dependent rate of order $\widetilde{\Omega} (1 / \sqrt{N (1 - \beta \gamma)})$ for the deviation in $\gamma$--discounted mean estimation, which is the first result that connects the statistical complexity of the problem with both the mixing property of the chain and the discount factor through the term $(1 - \beta\gamma)$. Thus, we can appreciate the role of the spectral gap $1-\beta$ in governing the complexity of the estimation problem.
It is worth noting that, when $\beta = 0$ and the chain mixes instantly making all the collected samples independent, the result reduces to H\"oeffding's rate~\citep{boucheron2013concentration} for independent random variables $\widetilde{\Omega} (1 / \sqrt{N})$. When $\gamma = 1$, it reduces to H\"oeffding's rate for general MCs $\widetilde{\Omega} (1 / \sqrt{N (1 - \beta) })$~\citep{fan2021hoeffding}. Notably, in the latter setting, the problem reduces to the estimation of the mean of a function under the stationary distribution $\pi$ of an MC. Finally, when $\beta = 1$ and the chain never mixes, the $\gamma$--discounted estimation problem is still well-defined for $\gamma < 1$. In Appendix~\ref{apx:proofs_minimax}, we report an additional result which shows that resetting the chain is indeed necessary in the latter no-mixing regime.

\section{ANALYSIS OF $\gamma$--DISCOUNTED MEAN ESTIMATORS}\label{sec:estimators}
In this section, we analyze four estimation algorithms $\mathfrak{A}=(\vrho,\widehat{\eta})$ for the $\gamma$--discounted mean $\pi_\gamma f$ from \emph{computational} and \emph{statistical} perspectives. We derive suitable concentration inequalities, compare the estimators and discuss their tightness \wrt the provided lower bound. We consider two classes of estimation algorithms, based on the nature of the reset policy: \emph{Fixed-Horizon Reset} (FHR, Section~\ref{sec:fixedHorizon}) and \emph{Adaptive-Horizon Reset} (AHR, Section~\ref{sec:adaptiveAnalysis}). Both classes of estimators assume the knowledge of the discount factor $\gamma$ but not of the absolute spectral gap $1-\beta$. Table~\ref{tab:summary} summarizes the properties of the presented estimators. The proofs of the results of this section are reported in Appendix~\ref{sec:apxProofs2}.

{\renewcommand{\arraystretch}{1.7}
\thinmuskip=1mu
\medmuskip=1mu
\thickmuskip=1mu
\begin{table*}[t]
\centering
\small
\begin{threeparttable}
\begin{tabular}{l|cc|C{2cm}C{2cm}C{3cm}}
\toprule
 \cellcolor{gray!50}& \multicolumn{2}{c}{\cellcolor{gray!50}\textbf{Computational properties}}  & \multicolumn{3}{c}{\cellcolor{gray!50}\textbf{Statistical properties}} \\
\cellcolor{gray!50}& \cellcolor{gray!25} &\cellcolor{gray!25}  &  \multicolumn{2}{c}{\cellcolor{gray!25}\emph{Concentration rate} \tnote{$\dagger$, $\ddagger$}}  & \cellcolor{gray!25} \\
\multirow{-3}{*}{\cellcolor{gray!50}\textbf{Estimator}} & \multirow{-2}{*}{\cellcolor{gray!25}\emph{\# parallel workers}} & \multirow{-2}{*}{\cellcolor{gray!25}\emph{Time complexity}\tnote{${}^*$}} & \cellcolor{gray!10} $\beta=0$ & \cellcolor{gray!10} $\beta=1$ & \multirow{-2}{*}{\cellcolor{gray!25}\emph{Minimax optimal}\tnote{$\mathsection$}} \\
\midrule
FHN & \multirow{2}{*}{$\displaystyle\frac{N}{T}$}  &  \multirow{2}{*}{$T$} & $\displaystyle \frac{1}{\sqrt{N}}$\tnote{$\mathparagraph$} &  $\displaystyle \frac{1}{\sqrt{N(1-\gamma)}}$\tnote{$\mathparagraph$} & $\begin{gathered} \text{\xmark}^\parallel \\ \text{(\cmark~for $\beta \in \{0,1\}$)} \end{gathered} $\\ 
\cdashline{1-1}[1pt/1pt]
\cdashline{4-6}[1pt/1pt]
FHC &  &  & $\displaystyle \frac{1}{\sqrt{N}}$\tnote{$\mathparagraph$} & $\displaystyle \frac{1}{\sqrt{N(1-\gamma)}}$\tnote{$\mathparagraph$} &  $\begin{gathered} \text{\xmark}^\parallel \\ \text{(\cmark~for $\beta \in \{0,1\}$)} \end{gathered} $\\ 
\cdashline{1-6}[1pt/1pt]
OS  & \multirow{2}{*}{$\begin{gathered} M \;\;\;\;{\footnotesize\text{       where}} \\ {\footnotesize \text{$M-1 \sim \mathrm{Bin}(N-1,1-\gamma)$}} \end{gathered}$} &  \multirow{2}{*}{$\displaystyle \min\bigg\{N, \frac{\log({N^2}/{\delta})}{1-\gamma} \bigg\}${}\small\tnote{$\ddagger$}}  & \multicolumn{2}{c}{$\displaystyle    \frac{1}{\sqrt{N(1-\gamma)}} $} & \xmark \\
\cdashline{1-1}[1pt/1pt]
\cdashline{4-6}[1pt/1pt]
AS  &  &  &  \multicolumn{2}{c}{$\displaystyle  \frac{1}{\sqrt{N(1-\beta\gamma)}} $} & \cmark \\ 
\bottomrule
\end{tabular}
\vspace{-.2cm}
\begin{multicols*}{3}
{\footnotesize
\begin{tablenotes}
\item[${}^*$] Big-$O$.
\item[$\dagger$] Big-$\widetilde{O}$.
\item[$\ddagger$] With probability at least $1-\delta$.
\item[$\mathsection$] According to our analysis.
\item[$\mathparagraph$] Selecting $T = O (\log N / \log (1/\gamma))$.
\item[$\parallel$] At least for $\beta \in (\overline{\beta},1)$ with $\overline{\beta} < 1$.
\end{tablenotes}
}
\end{multicols*}
\end{threeparttable}
\vspace{-.5cm}
\caption{Summary of the computational and statistical properties of the considered estimators.}\label{tab:summary}
\vspace{-.3cm}
\end{table*}
}

\subsection{Fixed-Horizon Estimation Algorithms}\label{sec:fixedHorizon}
The \emph{Fixed-Horizon} (FH) estimation algorithms perform a reset action  after having experienced a fixed number of transitions $T \in \Nat$, \ie generating trajectories with fixed \emph{horizon}. Thus, given $N$ transitions, the number of trajectories is given by $M = \lceil N / T\rceil$, with the last one possibly shorter $T_{M} = N-(M-1)T$. Thus, the reset policy takes the form  $\rho_t^{\text{FHR}}(\cdot|H_t,X_t) = \delta_{\mathbf{1}\{t \bmod T = 0\}}$. For the sake of the analysis, we assume that $N \bmod T = 0$, so that all trajectories have the same horizon $T$.

\subsubsection{Computational Properties}
The FHR reset policy $\vrho^{\text{FHR}}$ is easily parallelizable, as the horizon $T$ is known in advance. Thus, we need $M = N/T$ workers, each collecting a trajectory of $T$ samples, which also corresponds to the time complexity $O(T)$.
%Suppose we have $W \in \Nat$ parallel workers, then, each worker collects $\lceil M / W \rceil$ trajectories and, consequently, the time complexity of the sampling process is given by $\BigO(T \lceil M / W \rceil )$.

\subsubsection{Statistical Properties}
In the class of FH estimation algorithms, we consider two estimators $\widehat{\eta}$, the \emph{Fixed-Horizon Non-corrected} (FHN) and the \emph{Fixed-Horizon Corrected} (FHC) estimators:
\begin{align}
	& \widehat{\eta}_{\text{FHN}}(H_N,f) = \frac{1}{M} \sum_{i=0}^{M-1} \textcolor{black}{(1-\gamma)} \sum_{j=0}^{T-1} \gamma^j f(X_{Ti+j}), \\
	& \widehat{\eta}_{\text{FHC}}(H_N,f) = \frac{1}{M} \sum_{i=0}^{M-1}  \textcolor{black}{\frac{1-\gamma}{1-\gamma^T}} \sum_{j=0}^{T-1} \gamma^j f(X_{Ti+j}).
\end{align}
Both estimators are based on a sample average over the $M = N/T$ collected trajectories. The samples of each trajectory are weighted by the discount factor $\gamma$ raised to a suitable power. The difference between the two estimators lies in the coefficient that multiplies the inner summation. While in the FHN this constant disregards the fact that the summation is limited to the horizon $T$ employing $1-\gamma$ as normalizing constant, the FHC accounts for this by selecting the proper constant $\frac{1-\gamma^T}{1-\gamma} = \sum_{j=0}^{T-1} \gamma^j$. Nonetheless, as we shall see, it is not guaranteed that one estimator always outperforms the other in all regimes. These estimators are the most widely employed approaches for estimating $\gamma$--discounted means in RL~\citep[\eg][]{DeisenrothNP13, thomas2014bias, MetelliPR20}.

\textbf{Bias Analysis}~~As they truncate each trajectory after $T$ transitions, the FH estimators are affected by a bias, vanishing for large $T$, which is bounded as follows.

{\thinmuskip=1mu
\medmuskip=1mu
\thickmuskip=1mu
\begin{restatable}[FH Estimators -- Bias]{prop}{bias}\label{prop:biasFH}
	Let $H_N \sim P^N_{\nu,\textcolor{black}{\vrho}}$ with the reset policy $\rho_t = \delta_{ \indic \{ t \bmod \horizon = 0 \} } $, and let $f : \Xs \to [0, 1]$. The bias of the FHN and FHC estimators are upper bounded as:
	\begin{equation*}
	\resizebox{8.2cm}{!}{$\displaystyle
	\begin{aligned}
		&\Bias_{H_N \sim P_{\nu,\vrho}^N}[\widehat{\eta}_{\text{FHN}} (H_N, f)] \le b_{\gamma,T}^{\text{FHN}} \coloneqq \gamma^T, \\
		& \Bias_{H_N \sim P_{\nu,\vrho}^N}[\widehat{\eta}_{\text{FHC}} (H_N, f)] \le b_{\beta,\gamma, T}^{\text{FHC}} \coloneqq (1-\gamma) \gamma^T \min_{\substack{t_0 \in \dsb{0,T} \\ s_0 \in \Nat}}  \Bigg\{ \frac{2-\gamma^{t_0}-\gamma^{s_0}}{1-\gamma} \\
		& \quad + \left( \frac{(\beta\gamma)^{t_0}(1-(\beta\gamma)^{T-t_0})}{(1-\gamma^T)(1-\beta\gamma)} + \frac{(\beta\gamma)^{s_0}\beta^{T}}{1-\beta\gamma}\right) \sqrt{\chi_2(\nu \| \pi) \sigma^2 f} \Bigg\} .
	\end{aligned}$}
	\end{equation*}
\end{restatable}
}

Some observations are in order. First, both estimators are asymptotically unbiased as the horizon $T \rightarrow + \infty$. However, none of them is asymptotically unbiased as the budget $N \rightarrow +\infty$ (provided that $T$ does not depend on $N$). Second, the bias of the FHN estimator does not depend on the spectral gap $1-\beta$ of the underlying MC. This is expected, as the normalizing constant generates a scale inhomogeneity, regardless the mixing properties of the MC. Third, the bias of the FHC estimator, instead, depends on the absolute spectral gap $1-\beta$ and on the divergence $\chi_2(\nu \| \pi)$ between the initial-state distribution $\nu$ and the stationary distribution $\pi$. Thus, in the special case in which $\pi=\nu$ a.s., the bias of the FHC estimator vanishes. Nevertheless, the dependence on $\beta$ is quite convoluted, and the bound requires an optimization over the auxiliary integer variables $t_0$ and $s_0$. Although for the general case the optimization is non-trivial, for the extreme cases $\beta \in\{0,1\}$, we obtain more interpretable expressions that are reported in the following.

{\thinmuskip=1mu
\medmuskip=1mu
\thickmuskip=1mu
\begin{restatable}[FHC Estimator -- Bias]{coroll}{biasCoroll}
	Let $H_N \sim P^N_{\nu,\textcolor{black}{\vrho}}$ with the reset policy $\rho_t = \delta_{ \indic \{ t \bmod \horizon = 0 \} } $, and let $f : \Xs \to [0, 1]$. The bias of the FHC estimator is upper bounded as:
	\begin{itemize}[leftmargin=*,topsep=0pt,noitemsep]
		\item if $\beta = 0$:
		\begin{align}
			b_{0,\gamma, T}^{\text{FHC}} = \frac{1-\gamma}{1-\gamma^T} \gamma^T \min \left\{\sqrt{\chi_2(\nu \| \pi) \sigma^2 f} ,1-\gamma^T \right\};
		\end{align}
		\item if $\beta = 1$:
		\begin{align}
		b_{1,\gamma, T}^{\text{FHC}} = 2 \gamma^T \min\left\{\sqrt{\chi_2(\nu \| \pi) \sigma^2 f} , 1 \right\}.
		\end{align}
	\end{itemize}
\end{restatable}
}

Thus, when $\chi_2(\nu \| \pi) \sigma^2 f \gg 1$ the FHC estimator suffers a smaller bias than the FHN one when $\beta=0$, but, surprisingly larger by a factor $2$, when $\beta = 1$. Indeed, when the chain is slowly mixing ($\beta \approx 1$), both will deliver poor estimations and the FHN estimator mitigates this by using a smaller normalization constant. This result, which, to the best of our knowledge, has never appeared in the literature, justifies the use of the non-corrected estimator, especially when it is known that the underlying MC is slowly mixing.

%
%\begin{table}[t]
%{\thinmuskip=2mu
%\medmuskip=2mu
%\thickmuskip=2mu
%\centering
%\begin{tabular}{lcc}
%\toprule
%Estimator & $\beta = 0$ & $\beta=1$\\
%\midrule
%FHN & \multicolumn{2}{c}{$\gamma^T$} \\ 
%FHC &  $ (1-\gamma) \gamma^T \min \left\{\frac{c}{1-\gamma^T},1 \right\}$ &  $ 2 \gamma^T \min\{c, 1\}$  \\ 
%\bottomrule
%\end{tabular}
%}
%\caption{Bias bound of the finite-horizon corrected (FHC) and non-corrected (FHN) estimators varying the absolute spectral gap $1-\beta$, where $c = \sqrt{\chi_2(\nu \| \pi) \sigma^2 f}$.}\label{tab:biasFH}
%\end{table}

\textbf{Concentration Inequalities}~~Let us now move to the derivation of the concentration inequalities for the FH estimators. The technical challenge in this task consists in effectively exploiting the mixing properties of the underlying MC in order to derive tight concentration results. The following provides the general concentration result, which we particularize for specific values of $\beta$ later.

{\thinmuskip=2mu
\medmuskip=2mu
\thickmuskip=2mu
\begin{restatable}[FH Estimators -- Concentration]{thr}{concFH}\label{thr:concFH}
	Let $H_N \sim P^N_{\nu,\textcolor{black}{\vrho}}$ with the reset policy $\rho_t = \delta_{ \indic \{ t \bmod \horizon = 0 \} } $, and let $f : \Xs \to [0, 1]$. Let us define for $j_0 \in \dsb{0,T}$:
	\begin{align*}
		&c_{\beta,\gamma}(j_0) \coloneqq \frac{(\beta\gamma)^{j_0} - (\beta\gamma)^T}{1-\beta\gamma} \sqrt{\chi_2(\nu \|\pi) \sigma^2 f},\\
		&d_{\beta,\gamma}(j_0,\delta) \coloneqq \sqrt{ \frac{8T \left(\log \left(\chi_2(\nu P^{j_0} \| \pi) + 1 \right)  + 4 \log \frac{2}{\delta} \right)}{ N} }\\
		& \qquad\qquad \quad \times \sqrt{ \frac{ (1-\gamma^{j_0})^2}{(1-\gamma)^2} + \frac{(1+\beta)(\gamma^{2j_0} - \gamma^{2T})}{(1-\beta)(1-\gamma^2)} }.
	\end{align*}
	For every $\delta \in (0,1)$ with probability at least $1-\delta$, it holds:
	\begin{equation}
	\begin{aligned}
		& \left| \widehat{\eta}_{\text{FHN}} (H_N, f) - \pi_\gamma f \right|  \le  b_{\gamma,T}^{\text{FHN}} \\
		& \qquad\quad + (1-\gamma) \min_{j_0 \in \dsb{0,T}} \left\{c_{\beta,\gamma}(j_0) + d_{\beta,\gamma}(j_0,\delta) \right\}, 
	\end{aligned}
	\end{equation}
	\begin{equation}
	\begin{aligned}
	& \left| \widehat{\eta}_{\text{FHC}} (H_N, f) - \pi_\gamma f \right|  \le b_{\beta, \gamma,T}^{\text{FHC}} \\
		& \qquad\quad + \frac{1-\gamma}{1-\gamma^T} \min_{j_0 \in \dsb{0,T}} \left\{c_{\beta,\gamma}(j_0) + d_{\beta,\gamma}(j_0,\delta) \right\} .
	\end{aligned}
	\end{equation}
\end{restatable}
}
Similarly to the bias case, the resulting expression requires the optimization over a free variable $j_0 \in \dsb{0,T}$. Intuitively, $j_0$ should be selected (for analysis purpose only) as a function of $\beta$. Indeed, for slowly mixing chains ($\beta \approx 1$), we should select a small value of $j_0$ and vice versa. The following corollary provides the order of concentration for the extreme cases $\beta\in\{0,1\}$.

{\thinmuskip=1mu
\medmuskip=1mu
\thickmuskip=1mu
\begin{restatable}[FH Estimators -- Concentration]{coroll}{concFHcoroll}\label{coroll:FH}
	Let $H_N \sim P^N_{\nu,\textcolor{black}{\vrho}}$ with the reset policy $\rho_t = \delta_{ \indic \{ t \bmod \horizon = 0 \} } $, and let $f : \Xs \to [0, 1]$. Then, for any $\delta\in(0,1)$, with probability at least $1-\delta$, it holds that:\footnote{For interpretability reasons, we ignore the dependence on $\chi_2(\nu\|\pi)\sigma^2 f$.}
	 \begin{itemize}[leftmargin=*,topsep=0pt,noitemsep]
		\item if $\beta = 0$:
		\begin{equation}\resizebox{7cm}{!}{$\displaystyle
		\begin{aligned}
		&{ \left| \widehat{\eta}_{\text{FHN}} (H_N, f) - \pi_\gamma f \right|  \le O \left( \gamma^T +\sqrt{\frac{T (1-\gamma)  (1-\gamma^T)  \log \frac{2}{\delta}}{N}} \right) ,}
		\end{aligned}$}
		\end{equation}
		\begin{equation}\resizebox{7cm}{!}{$\displaystyle
		\begin{aligned}
		& \left| \widehat{\eta}_{\text{FHC}} (H_N, f) - \pi_\gamma f \right|  \le O \left( (1-\gamma)\gamma^T + \sqrt{\frac{T (1-\gamma) \log \frac{2}{\delta}}{N(1-\gamma^T)}} \right);
		\end{aligned}$}
		\end{equation}
		\item if $\beta = 1$:
		\begin{equation}
		\resizebox{7cm}{!}{$\displaystyle\begin{aligned}
		&\left| \widehat{\eta}_{\text{FHN}} (H_N, f) - \pi_\gamma f \right| \le O \left( \gamma^T + (1-\gamma^T) \sqrt{\frac{ T  \log \frac{2}{\delta}}{N}} \right), 
		\end{aligned}$}
		\end{equation}
		\begin{equation}
		\resizebox{6cm}{!}{$\displaystyle\begin{aligned}
	& \left| \widehat{\eta}_{\text{FHC}} (H_N, f) - \pi_\gamma f \right| \le O \left( \gamma^T + \sqrt{\frac{T \log \frac{2}{\delta}}{N}} \right).
		\end{aligned}$}
		\end{equation}
	\end{itemize}
\end{restatable}
}

We note that the FHC estimator outperforms (in the constants, but not in rate) the FHN when $\beta \approx 1$, whereas when $\beta \approx 0$, the FHN estimator enjoys better concentration.

\begin{remark}[About Minimax Optimality of the FH Estimators]
	A natural question, at this point, is whether the FH estimators match the minimax lower bound of Theorem~\ref{thr:minimax}. One could, in principle, select a value of the horizon $T$ depending on the spectral gap $1-\beta$ to tighten the confidence bounds. Unfortunately, $\beta$ is usually unknown in practice. Realistically, one should enforce a value of $T$ that depends on the discount factor $\gamma$, and, if necessary, on the confidence $\delta$, and the number of samples $N$. 
	
	{The FH estimators, according to our analysis, do not match the minimax lower bound for general $\beta$. When $\beta \in \{0,1\}$, we show in Appendix~\ref{sec:magic3} that the optimal $\beta$-independent choice of $T$ is $T^*_{\gamma} = O \left( (\log N)/(\log (1/\gamma)) \right)$, leading to the rate $\widetilde{O} ( 1/\sqrt{N} )$ for $\beta=0$ and $\widetilde{O} ( 1/\sqrt{N(1-\gamma)} )$ for $\beta = 1$, respectively.\footnote{Any choice of $T$ independent of $N$ (including the widely employed \quotes{effective horizon} $T = 1/(1-\gamma)$) will never lead to a consistent estimator, since the bias will not vanish as $N \rightarrow +\infty$.} In such regimes, both FH estimators nearly match the minimax lower bound. Nevertheless, in Appendix~\ref{sec:magicSec}, we show that there exists a regime of large values of $\beta$, namely $\beta \in (\overline{\beta},1)$ with $\overline{\beta} = (1+\gamma-2\gamma^T)/(1+\gamma-2\gamma^{T+1}) < 1$ for which the concentration rate is at least $\widetilde{\Omega} ( 1/\sqrt{N(1-\gamma)} )$ regardless the value of $\beta$ (when $0.3 \le \gamma < 1$), not matching the lower bound.}
	%This shows that these estimator do not match, according to our analysis, the minimax lower bound for at least one regime of $\beta$.}
	 %Instead, the FHC estimator nearly matches the minimax lower bound at least for the extreme cases $\beta \in\{0,1\}$, although we are currently unable to generalize for values $\beta \in (0,1)$. 
\end{remark}

\subsection{Adaptive-Horizon Estimation Algorithms}\label{sec:adaptiveAnalysis}
The \emph{Adaptive-Horizon} (AH) estimation algorithms generate trajectories with possibly different horizons. 
%Indeed, whether to reset the chain depends on a random event. 
At $t \in \Nat$, a Bernoulli random variable with parameter $1-\gamma$ is sampled, leading to the reset policy $\rho_t^{\text{AHR}}(H_t,X_t) = \mathrm{Ber}(1-\gamma)$. Thus, the horizon $T$ of each trajectory is a random variable too, where $T-1 \sim \mathrm{Geo}(1-\gamma)$ is a geometric distribution.\footnote{In a different perspective, one may first sample $T-1 \sim  \mathrm{Geo}(1-\gamma)$ and then simulate a trajectory of horizon $T$.} 

\subsubsection{Computational Properties}
In the AHR case, the parallel execution requires computing in advance the horizons $(T_i)_{i\in \dsb{M}}$ of each trajectory until we ran out of the sample budget $N$  and, subsequently, run in parallel the sample collection of each trajectory. From a technical perspective, characterizing the distribution of the individual $T_i$ is challenging. Indeed, since we need to stop as soon as we reach the budget $N$, the random variables $T_i$ become dependent. The following result characterizes the distribution of $M$ and the time complexity.
%, that is of the order of the longest horizon $\max_{i \in \dsb{M}} T_i$. 

\begin{restatable}[AH Estimators -- Complexity]{thr}{ahEstimatorComp}
Let $H_N \sim P^N_{\nu,\textcolor{black}{\vrho}}$ with the reset policy $\rho_t^{\text{AHR}}(H_t,X_t) = \mathrm{Ber}(1-\gamma)$. Then, the number of trajectories $M$ is distributed such that $M-1 \sim \mathrm{Bin}(N-1,1-\gamma)$. Furthermore, for every $\delta \in (0,1)$, with probability at least $1-\delta$, the time complexity is bounded as:
\begin{align*}
	\max_{ i\in \dsb{M}} T_{i} \le O \left( \min\left\{ N, \frac{\log \left( {N^2}/{\delta} \right)}{1-\gamma} \right\} \right).
\end{align*} 
\end{restatable}

Thus, the time complexity is a minimum between $N$, as no trajectory can be longer than the maximum number of transitions, and a term that grows with $\gamma$, since for large $\gamma$ the trajectories will have, on average, longer lengths.

%
%{\color{red}\begin{prop}
%
%Let $H_N$ be an history of $N$ samples collected with the reset policy $\rho_t^{\text{AHR}}(H_t,X_t) = \mathrm{Ber}(1-\gamma)$. Then, the number of trajectories $M$ is distributed such that $M-1 \sim \mathrm{Bin}(N-1,1-\gamma)$ with expectation is $\E[M] = 1+(N-1)(1-\gamma)$. Furthermore
%
%	The number of workers $M$ is distributed such that $M-1 \sim \mathrm{Bin}(N-1,1-\gamma)$, with expectation is $\E[M] = 1+(N-1)(1-\gamma)$ and the time complexity $T^{\max} = \max_{i \in \dsb{M}} T_i$, whose expectation is bounded as $\E[T^{\max}] \le 1 + \frac{1}{M} \sum_{k=1}^M \frac{1}{k}$.
%\end{prop}
%}

\subsubsection{Statistical Properties}
In the family of AH estimators, we analyze the concentration properties of two specific estimation algorithms: \emph{One-Sample} (OS) and \emph{All-Samples} (AS) estimators. 

\textbf{One-Sample Estimator}~~The idea behind the OS estimator is to regard the $\gamma$-discounted distribution $\pi_\gamma = \sum_{t \in \Nat} (1-\gamma)\gamma^t \nu P^t$ as the mixture of the distributions $\nu P^t$ with coefficients $(1-\gamma)\gamma^t$. The OS estimator offers a way of generating \emph{independent} samples from $\pi_\gamma$, by retaining the ones right before resetting is performed, \ie when $Y_t=1$:
\begin{align}\label{eq:OS}
	\widehat{\eta}_{\text{OS}}(H_N,f) = \frac{1}{M-1} \sum_{t=0}^{N-1} Y_t f(X_t).
	%, \\ \text{where}\qquad M - 1 = \sum_{t=0}^{N-1} \mathbf{1}\{Y_t = 1 \}. \nonumber
\end{align}

This estimator has been used in~\citep{thomas2014bias, metelli2021safe}, mostly for theoretical reasons, being unbiased. The following result provides the concentration.

\begin{restatable}[OS Estimator -- Concentration]{thr}{osEstimatorConc}
Let $H_N \sim P^N_{\nu,\textcolor{black}{\vrho}}$  with the reset policy $\rho_t^{\text{AHR}}(H_t,X_t) = \mathrm{Ber}(1-\gamma)$, and let $f : \Xs \to [0, 1]$. For every $\delta \in (0,1)$, with probability at least $1-\delta$, it holds that:
\begin{align*}
	\left| \widehat{\eta}_{\text{OS}}(H_N,f) - \pi_\gamma f \right| \le \sqrt{\frac{2 \log \frac{8}{\delta}}{N(1-\gamma)}}.
\end{align*}
\end{restatable}

The concentration term is governed by an \quotes{effective number of samples} that is $N(1-\gamma)$. Indeed, the probability of retaining each of the $N$ transitions is $1-\gamma$. It is worth noting that the concentration bound, as expected, does not depend on the absolute spectral gap $1-\beta$, since just one sample per trajectory is considered and, consequently, the estimators guarantees vanish as $\gamma \rightarrow 1$. Thus, this estimator is not minimax optimal, according to our analysis.

%{\color{red}Aggiungere comment}
%In particular, for $\gamma \rightarrow 1$ we obtain the asymptotic behavior $O \left( \exp(-\epsilon^2 (1-\gamma) N / 2) \right)$

\begin{figure*}[ht]
	\begin{subfigure}[t]{0.42\textwidth}
    		\centering
    		%\vspace{0.3cm}
    		\includegraphics[scale=1, valign=t]{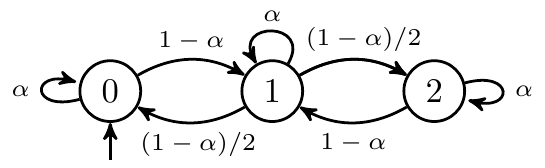}
    		%\vspace{0.2cm}
    		\caption{Illustrative MC}
    		\label{subfig:a}
    	\end{subfigure}
    	%\hfill
    	\begin{subfigure}[t]{0.57\textwidth}
    		\centering
    		\includegraphics[scale=1, valign=t]{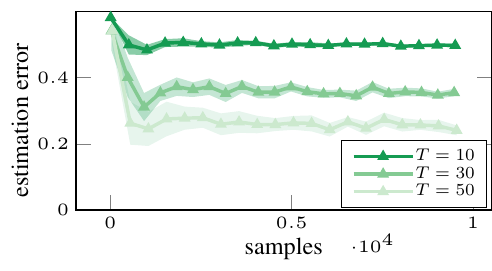}
    		\includegraphics[scale=1, valign=t]{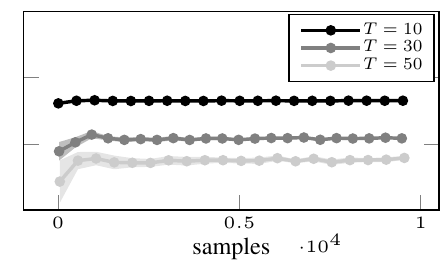}
    		\vspace{-0.3cm}
    		\caption{ \small \centering $\alpha = 0.99, \beta = 0.99, \gamma = 0.99$ }
    		\label{subfig:b}
    	\end{subfigure}
    	
    	\vspace{-0.6cm}
    	\begin{subfigure}[t]{0.475\textwidth}
    		\centering
    		\includegraphics[scale=1]{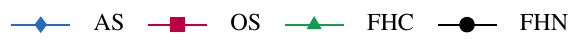}
    	\end{subfigure}
    	
    	\vspace{-0.15cm}
    	\begin{subfigure}[t]{0.2165\textwidth}
    		\centering
    		\includegraphics[scale=1, valign=t]{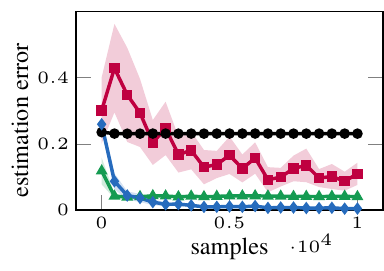}
    		\vspace{-0.3cm}
    		\caption{ \small \centering $\alpha = 0.005$ \newline \centering $\beta = 0.99, \gamma = 0.99$ }
    		\label{subfig:d}
    	\end{subfigure}
     \begin{subfigure}[t]{0.185\textwidth}
    		\centering
    		\includegraphics[scale=1, valign=t]{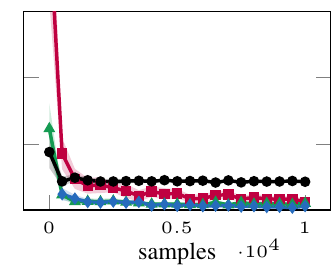}
    		\vspace{-0.3cm}
    		\caption{\small \centering $\alpha = 0.5$ \newline \centering $\beta = 0.5, \gamma = 0.9$}
    		\label{subfig:e}
    	\end{subfigure}
    	\begin{subfigure}[t]{0.185\textwidth}
    		\centering
    		\includegraphics[scale=1, valign=t]{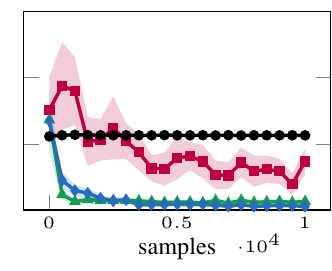}
    		\vspace{-0.3cm}
    		\caption{\small \centering $\alpha = 0.5$ \newline \centering $\beta = 0.5, \gamma = 0.99$}
    		\label{subfig:f}
    	\end{subfigure}
    	\begin{subfigure}[t]{0.185\textwidth}
    		\centering
    		\includegraphics[scale=1, valign=t]{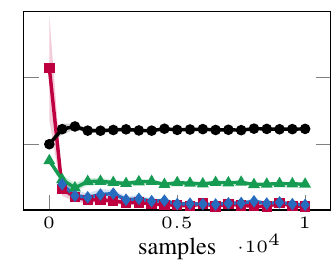}
    		\vspace{-0.3cm}
    		\caption{\small \centering $\alpha = 0.99$ \newline \centering $\beta = 0.99, \gamma = 0.9$}
    		\label{subfig:g}
    	\end{subfigure}
    	\begin{subfigure}[t]{0.18\textwidth}
    		\centering
    		\includegraphics[scale=1, valign=t]{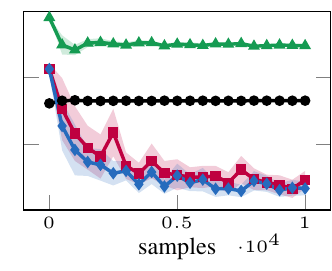}
    		\vspace{-0.3cm}
    		\caption{\small \centering $\alpha = 0.99$ \newline \centering $\beta = 0.99, \gamma = 0.99$}
    		\label{subfig:h}
    	\end{subfigure}
    	\vspace{-0.2cm}
    	\caption{$\gamma$--discounted mean estimation of the function $f(x) = (1, -1, 2)$ over the MC depicted in \textbf{(a)}. For each combination of parameter $\alpha$ and discount $\gamma$, we report the estimation error of the OS, AS, FHC, FHN estimators \textbf{(c, d, e, f, g)}. For the FHC, FHN, we provide a finer analysis on the impact of $T$ \textbf{(b)}. We report average and 95\% c.i. over 20 runs. }
    \label{fig:numerical_validation}
\end{figure*}

\textbf{All-Samples Estimator}~~
The AS estimator, instead, makes use of all the samples collected from the simulation. Clearly, this choice introduces a new trade-off since we have at our disposal a larger number of samples for estimation, but, unfortunately, within a single trajectory such samples are statistically dependent. Nevertheless, such a dependence is controlled by the mixing properties of the MC. The AS estimator takes the following form:
\begin{align}\label{eq:AS}
	\widehat{\eta}_{\text{AS}}(H_N,f) = \frac{1}{N} \sum_{t=0}^{N-1} f(X_t).
\end{align}

This estimator has been employed in~\citep{Konda02, XuWL20, EldowaBR22}.
The following result provides a concentration inequality for the AS estimator that highlights the dependence on the mixing properties.

\begin{restatable}[AS Estimator -- Concentration]{thr}{asEstimatorConc}
Let $H_N \sim P^N_{\nu,\textcolor{black}{\vrho}}$ with the reset policy $\rho_t^{\text{AHR}}(H_t,X_t) = \mathrm{Ber}(1-\gamma)$, and let $f : \Xs \to [0, 1]$. For every $\delta \in (0,1)$, with probability at least $1-\delta$, it holds that:
\begin{align*}
	\left| \widehat{\eta}_{\text{AS}}(H_N,f) - \pi_\gamma f \right| \le \sqrt{\frac{8 \log \frac{2}{\delta} + 4\log (\chi_2(\nu \| \pi_\gamma)+1)}{N(1-\beta\gamma)}}.
\end{align*}
\end{restatable}

We note the dependence with the spectral gap in $1-\beta\gamma$. Contrary to the OS estimator, the bound holds even for $\gamma\rightarrow 1$. We also observe a logarithmic dependence on the term $\chi_2(\nu \| \pi_\gamma)$ due to the small bias introduced by the sampling procedure. Most importantly, the AS estimator nearly matches the minimax lower bound of Theorem~\ref{thr:minimax}.
%\am{Possiamo dire qualcosa su $\chi_2(\nu \| \pi_\gamma)$ e $\chi_2(\nu \| \pi)$? Sicuramente:
%$$\chi_2(\nu \| \pi_\gamma) \le \frac{1}{1-\gamma}$$
%}

\section{NUMERICAL VALIDATION}
\label{sec:validation}

In this section, we confront the $\gamma$-discounted mean estimators presented in Section~\ref{sec:estimators} through numerical simulations to both support and complement the analysis of their statistical properties. To this end, we consider an illustrative family of MCs parametrized by $\alpha$ (Figure~\ref{subfig:a}). The parameter $\alpha$ allows controlling the mixing properties of the chain. If we set $\alpha$ close to either $0$ or $1$ we get a slow-mixing chain (large $\beta$), in which every state is nearly transient or absorbing respectively. $\alpha$ close to $\nicefrac{1}{2}$ gives a fast-mixing chain instead (small $\beta$). In this setting, we consider the problem of estimating the mean of the function $f(x) = (1, -1, 2)$ in different discounting regimes, namely $\gamma \in \{ 0.9, 0.99 \}$, where the performance of each estimator $\widehat{\eta}$ is measured in terms of the corresponding estimation error $|\widehat{\eta} - \pi_\gamma f|$.

In Figure~\ref{fig:numerical_validation}, we report the results of the numerical analysis. As a further testament of its compelling statistical properties, the AS estimator dominates the alternatives, achieving the smallest estimation error in every mixing-discounting regime. Interestingly, the unbiased OS estimator fails to quickly converge to the true mean with high discounting (Figures~\ref{subfig:d}~\ref{subfig:f}~\ref{subfig:h}), which is likely caused by its inherently large variance. The finite-horizon estimators FHC and FHN show a significant bias instead, despite an overall stable behavior. Although the corrected estimator FHC outperforms (as expected) the non-corrected FHN in most of the regimes (see Figures~\ref{subfig:d}-\ref{subfig:g}), the correction can skyrocket the bias in some unfortunate settings (see Figure~\ref{subfig:h}). This is particularly underwhelming as we cannot trust FHC as a default option even when committed to a finite-horizon estimation algorithm. Finally, in Figure~\ref{subfig:b} we provide a finer analysis of the finite-horizon estimators for different horizons $T$. Unsurprisingly, in a slow-mixing regime (note $\beta = 0.99$), increasing $T$ benefits the overall quality of the estimates for both FHC and FHN, as the bias is visually reduced at the cost of a slightly increased instability.

\section{CONCLUSIONS}
%In this paper, we have studied the problem of estimating the expectation of a function under the $\gamma$--discounted stationary distribution in a Markov chain. This problem arises in several circumstances in reinforcement learning, although the literature did not focus on it significantly. We have first formulated this problem with a general and flexible framework, based on the notion of reset policy. Then, we have analyzed the intrinsic complexity of the problem, deriving a minimax lower bound. This represents, to the best of our knowledge, the first result of this nature for the problem of estimating $\gamma$--discounted means. We have then considered two classes of estimation algorithms, based on different reset policy approaches. For both of them, we have provided a computational analysis and a concentration inequality. In particularity, we have shown that the all-sample estimator is able to achieve the best statistical profile, matching the derived lower bound. Finally, our numerical validation has provided evidence on the properties of the considered estimators.
%\am{Da rivedere}

In this paper, we have studied the problem of estimating the mean of a function under the $\gamma$--discounted stationary distribution of an MC. We have formulated this problem with a general and flexible framework and then analyzed its intrinsic complexity through a minimax lower bound. Finally, we have considered two classes of estimation algorithms, for which we provided a study of computation and statistical properties, as well as a numerical validation.

The aim of this paper is far from being a theoretical detour, as we believe that our contribution has significant practical implications in discounted RL. Especially, the all-samples estimator resulted in the best statistical profile among the considered alternatives, while still supporting parallel sampling. This signals an avenue to develop improved ``deep reinforcement learning"~\citep{franccois2018introduction} algorithms based on sampling from the discounted kernel. Other interesting future directions include the study of the $\gamma$-discounted mean estimation for inhomogeneous functions, which is akin to the practical implementations of Q-learning~\citep{watkins1992q}, the estimation of other functionals beyond the expectation~\citep{chandak2021universal}, and extending our analysis to generalized notions of discount~\citep{yoshida2013reinforcement, franccois2015discount, pitis2019rethinking, fedus2019hyperbolic, tang2021taylor}.

Finally, our results can be of independent interest in the MC literature while bridging fundamental problems in discounted RL and concentration inequalities for MCs~\citep{samson2000concentration, glynn2002hoeffding, leon2004optimal, kontorovich2008concentration, paulin2015concentration}.

\clearpage
%%%%%%%%%%%%%%%%%%%%%%%%%%%%%%%%%%%%%%%%%%%%%%%%%%%%%%%%%%%%
\bibliography{biblio}
%%%%%%%%%%%%%%%%%%%%%%%%%%%%%%%%%%%%%%%%%%%%%%%%%%%%%%%%%%%%

%%%%%%%%%%%%%%%%%%%%%%%%%%%%%%%%%%%%%%%%%%%%%%%%%%%%%%%%%%%%

\clearpage
\onecolumn
\appendix

\setlength{\abovedisplayskip}{8pt}
\setlength{\belowdisplayskip}{8pt}

\section{PROOFS}
\label{apx:proofs}

\subsection{Proofs of Section~\ref{sec:minimax_lower_bound}}
\label{apx:proofs_minimax}

\begingroup\renewcommand\footnote[1]{}
\minimax*
\endgroup
\begin{proof}
The proof is articulated in four steps.
\paragraph{First step: Chain construction}
	We consider a 2-states MC with state space $\Xs=\{A,B\}$ whose kernel is parametrized via $p$ and $\beta$:
	\begin{multicols}{2}
	\begin{align*}
	P=
		\begin{pmatrix}
			p+\beta & 1-p-\beta \\
			p & 1-p
		\end{pmatrix},
	\end{align*}
	
	\begin{tikzpicture}[->, >=stealth', auto, semithick, node distance=3cm]
\tikzstyle{every state}=[fill=white,draw=black,thick,text=black,scale=1]
\node[state]    (A)             {$A$};
\node[state]    (B)[right of=A]   {$B$};
\path
(A) edge[loop left]     node{$p+\beta$}         (A)
    edge[bend left,above]      node{$1-p-\beta$}      (B)
(B) edge[loop right]    node{$1-p$}     (B)
    edge[bend left,below]     node{$p$}         (A);
\end{tikzpicture} 
	\end{multicols}
	where $\beta \in [0,1)$ and $p \in (0, 1-\beta)$. The kernel $P$ admits eigenvalues $\{1,\beta\}$ and has a unique invariant measure $\pi = \left( \frac{p}{1-\beta}, 1- \frac{p}{1-\beta}\right)$. We immediately verify that $\Pi^T \odot P$ is symmetric, where $\Pi = \mathbf{1}\pi$ and $\odot$ denotes the Hadamard product. Consequently, the MC is reversible and, thus, its spectral gap is $1-\beta$. We consider as initial state distribution $\nu=(q,1-q)$ parametrized by $q \in (0,1)$. The specific values of $p$ and $q$ will be specified later in the proof. 
	
	Consider two functions $f_1, f_{-1}: \Xs \in [-1,1]$ defined as: $f_{s}(A) = s$ and $f_s(B) = 0$ for $s \in \{-1,1\}$. Simple calculations allow to show that the corresponding expectations and variances, under the $\gamma$-discounted stationary distributions are given for $s \in \{-1,1\}$ by:
	\begin{align*}
	& \pi_\gamma f_s = \sum_{x \in \Xs} \pi_\gamma(x) f_s(x) =  s\frac{(1-\gamma)q+\gamma p}{1-\beta \gamma}, \\
	& \sigma_f^2 f_s = \pi_\gamma(f_s - \mathbf{1} \pi_\gamma f)^2 = \pi_\gamma f_s (1 - \pi_\gamma f_s).
	\end{align*}
	
	\paragraph{Second step: Lower bounding the probability of deviation}
	We now proceed at lower bounding the probability of making an error larger than $\epsilon$, with $\epsilon \in [0,1]$. The intuition is that we need to compute the probability not to distinguish the two instances, \ie never visiting state $A$.

	For any values of $p$ and $q$ such that $\pi_\gamma f_1 = \epsilon$, considering a generic estimator $(\textcolor{black}{\vrho}, \textcolor{black}{\widehat{\eta})}$, we have:
	\begin{align}
	\sup_{\substack{P,\nu,f \\ \text{with spectral gap $\beta$}}} & \Prob_{H_N \sim P^N_{\nu,\textcolor{black}{\vrho}} }\left( \left| \textcolor{black}{\widehat{\eta}}(H_{N},f) - \pi_\gamma f \right| \geq \epsilon \right) \ge \max_{s \in \{-1,1\}} \Prob_{H_N \sim P^N_{\nu,\textcolor{black}{\vrho}} }\left( \left| \textcolor{black}{\widehat{\eta}}(H_{N},f_s) - \pi_\gamma f_s \right| \geq \epsilon \right) \\
	& \ge \frac{1}{2}  \left( \Prob_{H_N \sim P^N_{\nu,\textcolor{black}{\vrho}} }\left( \left| \textcolor{black}{\widehat{\eta}}(H_{N},f_{-1}) - \pi_\gamma f_{-1} \right| \geq \epsilon \right) + \Prob_{H_N \sim P^N_{\nu,\textcolor{black}{\vrho}} }\left( \left| \textcolor{black}{\widehat{\eta}}(H_{N}f_1) - \pi_\gamma f_1 \right| \geq \epsilon \right)\right) \label{p:001} \\
	& \ge \frac{1}{2}  \Prob_{H_N \sim P^N_{\nu,\textcolor{black}{\vrho}} }\left( \left| \textcolor{black}{\widehat{\eta}}(H_{N},f_{-1}) - \pi_\gamma f_{-1} \right| \geq \epsilon \vee \left| \textcolor{black}{\widehat{\eta}}(H_{N},f_{1}) - \pi_\gamma f_1 \right| \geq \epsilon \right) \label{p:002}\\
	& \ge \frac{1}{2} \Prob_{H_N \sim P^N_{\nu,\textcolor{black}{\vrho}} }  \left(  \textcolor{black}{\widehat{\eta}}(H_{N},f_{-1}) =  \textcolor{black}{\widehat{\eta}}(H_{N},f_{1})\right) \label{p:003}\\
	& \ge \frac{1}{2}  \Prob_{H_N \sim P^N_{\nu,\textcolor{black}{\vrho}} } \left( \forall t \in \{0,\dots, N-1\}\,:\, X_{t} = B \right) \label{p:004} \\
	& = \frac{1}{2} \nu(\{B\}) \min \left\{ \nu(\{B\}), P(\{B\}|B) \right\}^{N-1} \ge \frac{1}{2}  \min \{1-q,1-p\}^{N}, \label{p:005}
	\end{align}
	where in line~\eqref{p:001}  we exploited the inequality $\max\{x,y\} \ge \frac{1}{2}(x+y)$, in line~\eqref{p:002} we employed a union bound, in line~\eqref{p:003} we used the fact that $\pi_\gamma f_s = s \epsilon$ by assumption and consequently the event $\{ \textcolor{black}{\widehat{\eta}}(H_{N},f_{-1}) =  \textcolor{black}{\widehat{\eta}}(H_{N},f_{1})\}$ is included in the event $\{\left| \textcolor{black}{\widehat{\eta}}(H_{N},f_{-1}) - \pi_\gamma f_{-1} \right| > \epsilon \} \cup \{\left| \textcolor{black}{\widehat{\eta}}(H_{N},f_{1}) - \pi_\gamma f_{-1} \right| > \epsilon \}$. In line~\eqref{p:004} comes from the observation that in order to have equal values of the estimators we must not distinguish the two chain instances, that in turn happens only when we never visit state $A$. Finally, line~\eqref{p:005} follows by taking the minimum probability for never landing to state $A$ depending on both reset and transition probability.

	For any values of $p$ and $q$ such that $\pi_\gamma f_1 = \epsilon$, considering a generic estimator $(\vrho, \widehat{\eta})$, we have:
	\begin{align}
	\sup_{\substack{P,\nu,f \\ \text{with spectral gap $\beta$}}} & \Prob_{H_N \sim P^N_{\nu,\vrho }}\left( \left| \widehat{\eta}(H_{N},f) - \pi_\gamma f \right| \geq \epsilon \right) \ge \max_{s \in \{-1,1\}} \Prob_{H_N \sim P^N_{\nu,\vrho} }\left( \left| \widehat{\eta}(H_{N},f_s) - \pi_\gamma f_s \right| \geq \epsilon \right) \\
	& \ge \frac{1}{2}  \left( \Prob_{H_N \sim P^N_{\nu,\vrho} }\left( \left| \widehat{\eta}(H_{N},f_{-1}) - \pi_\gamma f_{-1} \right| \geq \epsilon \right) + \Prob_{H_N \sim P^N_{\nu,\vrho}}\left( \left| \widehat{\eta}(H_{N}f_1) - \pi_\gamma f_1 \right| \geq \epsilon \right)\right) \label{p:001} \\
	& \ge \frac{1}{2}  \Prob_{H_N \sim P^N_{\nu,\vrho} }\left( \left| \widehat{\eta}(H_{N},f_{-1}) - \pi_\gamma f_{-1} \right| \geq \epsilon \vee \left| \widehat{\eta}(H_{N},f_{1}) - \pi_\gamma f_1 \right| \geq \epsilon \right) \label{p:002}\\
	& \ge \frac{1}{2} \Prob_{H_N \sim P^N_{\nu,\vrho} }  \left(  \widehat{\eta}(H_{N},f_{-1}) =  \widehat{\eta}(H_{N},f_{1})\right) \label{p:003}\\
	& \ge \frac{1}{2}  \Prob_{H_N \sim P^N_{\nu,\vrho} } \left( \forall t \in \{0,\dots, N-1\}\,:\, X_{t} = B \right) \label{p:004} \\
	& = \frac{1}{2} \nu(\{B\}) \min \left\{ \nu(\{B\}), P(\{B\}|B) \right\}^{N-1} \ge \frac{1}{2}  \min \{1-q,1-p\}^{N}, \label{p:005}
	\end{align}
	where in line~\eqref{p:001}  we exploited the inequality $\max\{x,y\} \ge \frac{1}{2}(x+y)$, in line~\eqref{p:002} we employed a union bound, in line~\eqref{p:003} we used the fact that $\pi_\gamma f_s = s \epsilon$ by assumption and consequently the event $\{ \widehat{\eta}(H_{N},f_{-1}) =  \widehat{\eta}(H_{N},f_{1})\}$ is included in the event $\{\left| \widehat{\eta}(H_{N},f_{-1}) - \pi_\gamma f_{-1} \right| > \epsilon \} \cup \{\left| \widehat{\eta}(H_{N},f_{1}) - \pi_\gamma f_{-1} \right| > \epsilon \}$. In line~\eqref{p:004} comes from the observation that in order to have equal values of the estimators we must not distinguish the two chain instances, that in turn happens only when we never visit state $A$. Finally, line~\eqref{p:005} follows by taking the minimum probability for never landing to state $A$ depending on both reset and transition probability.
	
	\paragraph{Third step: Tightening the bound} Now, we need to compute the values of $p$ and $q$ in order to make the bound as tight as possible while fulfilling all the constraints. This leads to the optimization problem:
	\begin{align*}
		& \max_{p,q} \min\{1-q,1-p\}\\
		& \text{s.t. } 0 < p < 1-\beta \\
		& \phantom{\text{s.t. }} 0<q < 1 \\
		& \phantom{\text{s.t. }} \frac{(1-\gamma)q+\gamma p}{1-\beta \gamma} = \epsilon.
	\end{align*}
	First of all, we exploit the constraint with equality to express $q$ as a function of $p$, \ie $q = \frac{\epsilon(1-\beta\gamma)}{1-\gamma}- \frac{\gamma p}{1-\gamma}$. Now, we consider the two cases:
	
	\underline{Case 1: $p \le q$}~~In this case, the $\min$ in the objective function reduces to $1-q = 1 - \frac{\epsilon(1-\beta\gamma)}{1-\gamma}+ \frac{\gamma p}{1-\gamma}$ that is maximized by taking the maximum value of $p$ fulfilling the constraints:
	\begin{align*}
	& \max_{p} 1 - \frac{\epsilon(1-\beta\gamma)}{1-\gamma}+ \frac{\gamma p}{1-\gamma}\\
	& \text{s.t. } 0 < p < 1-\beta \\
	&	\phantom{\text{s.t. }} 0 < q < 1 \implies 0<\frac{\epsilon(1-\beta\gamma)}{1-\gamma}- \frac{\gamma p}{1-\gamma} < 1 \\
	&	\phantom{\text{s.t. }} p \le q \implies p \le \frac{\epsilon(1-\beta\gamma)}{1-\gamma}- \frac{\gamma p}{1-\gamma}
	\end{align*}
	This leads to:
	\begin{align*}
		p = \begin{cases}
			\epsilon(1-\beta\gamma) & \text{if } \epsilon \in \left[ 0, \frac{1-\beta}{1-\beta\gamma}\right] \\
			1-\beta & \text{if } \epsilon \in \left(\frac{1-\beta}{1-\beta\gamma}, 1\right]
		\end{cases} \quad\implies\quad 1-q = \begin{cases}
			1-\epsilon(1-\beta\gamma) & \text{if } \epsilon \in \left[ 0, \frac{1-\beta}{1-\beta\gamma}\right] \\
			\frac{(1-\beta\gamma)(1-\epsilon)}{1-\gamma} & \text{if } \epsilon \in \left(\frac{1-\beta}{1-\beta\gamma}, 1\right]
		\end{cases}.
	\end{align*}
	
	\underline{Case 2: $p > q$}~~In this case, the $\min$ in the objective function reduces to $1-p$ that is maximized by taking the minimum value of $p$ fulfilling the constraints:
		\begin{align*}
	& \max_{p}1-p\\
	& \text{s.t. } 0 < p < 1-\beta \\
	&	\phantom{\text{s.t. }} 0 < q < 1 \implies 0<\frac{\epsilon(1-\beta\gamma)}{1-\gamma}- \frac{\gamma p}{1-\gamma} < 1 \\
	&	\phantom{\text{s.t. }} p > q \implies p > \frac{\epsilon(1-\beta\gamma)}{1-\gamma}- \frac{\gamma p}{1-\gamma}
	\end{align*}
	The problem is feasible only when $\epsilon \in \left[ 0, \frac{1-\beta}{1-\beta\gamma}\right]$. In such a case, we have:
	\begin{align*}
		p = \epsilon(1-\beta\gamma) \quad\implies\quad 1 - q = 	1-\epsilon(1-\beta\gamma) \quad \text{if } \epsilon \in \left[ 0, \frac{1-\beta}{1-\beta\gamma}\right]	.
	\end{align*}
	
	\paragraph{Fourth step: Algebraic manipulation}
	We now proceed at performing some manipulation to get more interpretable result. In the small-$\epsilon$ regime, we have:
	\begin{align}
	\left( 1- \epsilon(1-\beta\gamma) \right)^{N} & \ge \exp \left(- \frac{\epsilon N (1-\beta\gamma)}{1-\epsilon(1-\beta\gamma)} \right) \label{dev:001} \\ 
	& = \exp \left( -\frac{\epsilon^2 N (1-\beta\gamma)}{\sigma^2_\gamma f} \cdot \frac{1-\epsilon}{1-\epsilon(1-\beta\gamma)}\right) \label{dev:002}\\
	& \ge \exp \left( -\frac{\epsilon^2 N (1-\beta\gamma)}{\sigma^2_\gamma f}\right),
	\end{align}
	where we exploited the inequality $1 - x \geq \exp (- x / (1 - x))$ in line~\eqref{dev:001}, and the fact that $\sigma^2 f = \epsilon (1 - \epsilon)$ in line~\eqref{dev:002}.
	Following similar steps for the large-$\epsilon$ regime, we have:
	\begin{align*}
		\left( \frac{(1-\beta\gamma)(1-\epsilon)}{1-\gamma} \right)^N & \ge \exp \left( - N \frac{1- \frac{(1-\beta\gamma)(1-\epsilon)}{1-\gamma}}{\frac{(1-\beta\gamma)(1-\epsilon)}{1-\gamma}} \right) \\
		& = \exp \left( -  \frac{N \epsilon \left(1- \frac{(1-\beta\gamma)(1-\epsilon)}{1-\gamma}\right)}{\sigma^2_\gamma f} \cdot \frac{1-\epsilon}{\frac{(1-\beta\gamma)(1-\epsilon)}{1-\gamma}} \right)  \\
		% & =\exp \left( -  \frac{N \epsilon \left(1 - \gamma -(1-\beta\gamma)(1-\epsilon)\right)}{\sigma^2_\gamma f \cdot (1-\beta\gamma)}  \right)\\
		& =\exp \left( -  \frac{N \epsilon \left( \epsilon(1-\beta \gamma) - \gamma (1-\beta)\right)}{\sigma^2_\gamma f \cdot (1-\beta\gamma)}  \right)\\
		& \ge \exp \left( -  \frac{N \epsilon^2 }{\sigma^2_\gamma f }  \right).
	\end{align*}	
	Finally, we can reformulate the previous results on the confidence $\delta$ in terms of deviation $\epsilon$, such that we have with probability at least $1 - \delta$
	\begin{align*}
		\inf_{\vrho, \widehat{\eta}} \sup_{\substack{P,\nu,f \\ \text{with spectral gap $1-\beta$}}} 
		\left| \widehat{\eta}(H_{N},f) - \pi_\gamma f \right| \geq 
		\begin{cases}
		\sqrt{\frac{\sigma^2_\gamma f \log \frac{1}{2\delta}}{N (1 - \beta\gamma)}} & \text{if } \delta \in \left( 0,   \frac{1}{2} \exp \left(- \frac{N (1 - \beta)^2}{ \sigma^2_\gamma f (1 - \beta \gamma)} \right)  \right) \\
		 \sqrt{\frac{\sigma^2_\gamma f \log \frac{1}{2\delta}}{N}} & \text{otherwise}
		\end{cases},
	\end{align*}
	which concludes the proof.
For the sake of clarity, we only report the most meaningful high-confidence regime in the theorem statement.
\end{proof}

The following result shows that the reset is unavoidable, at least, for the case in which the underlying MC does not mix, \ie when $\beta = 1$.

\begin{thr}[Reset is Unavoidable]
For any non-reset policy, \ie  $\vrho = (\rho_t)_{t \in \Nat}$ such that $\rho(\cdot|H_t,X_t) = \delta_{0}(\cdot)$ it holds:
\begin{align*}
\inf_{\widehat{\eta}} \sup_{\substack{P,\nu,f \\ \text{with spectral gap $1-\beta$}}} \Prob_{H_N \sim P^N_{\nu,\textcolor{black}{\delta_0}} }\left( \left| \textcolor{black}{\widehat{\eta}}(H_{N},f) - \pi_\gamma f \right| > \frac{1}{2} \right) \ge \frac{1}{4} \left( \frac{1+\beta}{2} \right)^{N-1}.
\end{align*}
\end{thr}

\begin{proof}
	We consider a 2-states MC $\Xs=\{A,B\}$, with kernel:
	\begin{align*}
		P= \begin{pmatrix}
			\frac{1+\beta}{2} & \frac{1-\beta}{2} \\
			 \frac{1-\beta}{2} & \frac{1+\beta}{2}
		\end{pmatrix}.
	\end{align*}
	It is easy to see that the spectral gap is $\beta$. Consider two functions $f_1, f_{-1}: \Xs \in [-1,1]$ defined as: $f_{s}(A) = s$ and $f_s(B) = 0$ for $s \in \{-1,1\}$. Consider the initial state distribution $\nu = (1/2,1/2)$. It is simple to show that $\pi_\gamma = (1/2,1/2)$ and, consequently $\pi_\gamma f_{s} = s/2$ for $s \in \{-1,1\}$. Consider now a non-reset policy, it holds that:
	\begin{align*}
	 \sup_{\substack{P,\nu,f \\ \text{with  spectral gap $1-\beta$}}}  & \Prob_{H_N \sim P^N_{\nu,\delta_0} }\left( \left| \widehat{\eta}(H_{N}) - \pi_\gamma f \right| > \frac{1}{2} \right) \\
	 &\ge \max_{f_s : s \in \{-1,1\}} \Prob_{H_N \sim P^N_{\nu,\delta_0} }\left( \left| \widehat{\eta}(H_{N}) - \pi_\gamma f_s \right| > \frac{1}{2} \right) \\
	&\ge \frac{1}{2}  \left( \Prob_{H_N \sim P^N_{\nu,\delta_0} }\bigg( \left| \widehat{\eta}(H_{N}) - \pi_\gamma f_{-1} \right| > \frac{1}{2} \right) 
	+ \Prob_{H_N \sim P^N_{\nu,\delta_0} }\left( \left| \widehat{\eta}(H_{N}) - \pi_\gamma f_1 \right| > \frac{1}{2} \right)\bigg) \\
	&\ge \frac{1}{2}  \Prob_{H_N \sim P^N_{\nu,\delta_0} }\bigg( \left| \widehat{\eta}(H_{N},f_{-1}) - \pi_\gamma f_{-1} \right| > \frac{1}{2}
	\vee \left| \widehat{\eta}(H_{N},f_{1}) - \pi_\gamma f_1 \right| > \frac{1}{2} \bigg) \\
	& \ge \frac{1}{2} \Prob_{H_N \sim P^N_{\nu,\delta_0} }  \left(  \widehat{\eta}(H_{N},f_{-1}) =  \widehat{\eta}(H_{N},f_{1})\right) \\
	&\ge \frac{1}{2}  \Prob_{H_N \sim P^N_{\nu,\delta_0}} \left( \forall t \in \{0,\dots, N-1\}\,:\, X_{t} = B \right) = \frac{1}{4} \left( \frac{1+\beta}{2} \right)^{N-1}.
	\end{align*}
\end{proof}
The latter result implies that, when $\beta=1$ and the chain never mixes, we have
 $$\inf_{\vrho, \widehat{\eta}} \sup_{\substack{P,\nu,f \\ \text{with spectral gap $0$}}} \Prob_{H_N \sim P^N_{\nu,\delta_0} }\left( \left| \widehat{\eta}(H_{N},f) - \pi_\gamma f \right| > \frac{1}{2} \right) \ge \frac{1}{4},$$ 
showing that reset is actually necessary.

\subsection{Proofs of Section~\ref{sec:estimators}}\label{sec:apxProofs2}

\subsubsection{Fixed-Horizon Estimation Algorithms}

\paragraph{Bias Analysis}

\bias*

\begin{proof}
	Let us start from the bias of the FHN estimator. We proceed as follows, with $t_0 \in \mathbb{N}$:
	\begin{align}
		\Bias_{H_N \sim P_{\nu,\vrho}^N}[\widehat{\eta}_{\text{FHN}} (H_N, f)]
		&= \left| \E_{H_N \sim P_{\nu, \rho}^N} \left[  \widehat{\eta}_{\text{FHN}} (H_N, f) \right] - \pi_\gamma f \right| \\
		&= (1 - \gamma) \left| \sum_{t = 0}^{\horizon - 1} \gamma^t \E_{X \sim \nu P^t} [ f (X) ] - \sum_{t = 0}^{+\infty} \gamma^t \E_{X \sim \nu P^t} [f (X)] \right| \label{eq:fhn1} \\
		&\leq (1 - \gamma) \gamma^\horizon \Bigg( \sum_{t = 0}^{t_0} \gamma^t \left| \E_{X \sim \nu P^{t + \horizon} } [ f (X)] \right| \label{eq:fhn2} \\
		&\quad+ \sum_{t = t_0}^{+\infty}\gamma^t \left| \E_{X \sim \nu P^{t + \horizon}} [f(X)] - \E_{X \sim \pi} [f (X)] \right| + \sum_{t = t_0}^{+\infty} \gamma^t \left| \E_{X \sim \pi} [f(X)] \right| \Bigg) \\
		&\leq (1 - \gamma) \gamma^\horizon \left( \frac{1 - \gamma^{t_0}}{1 - \gamma} + \sqrt{\chi_2 (\nu \| \pi) \sigma^2 f} \sum_{t = t_0}^{+\infty} \gamma^t \beta^{t + \horizon} + \frac{\gamma^{t_0}}{1 - \gamma}  \right) \label{eq:fhn3} \\
		&\leq (1 - \gamma) \gamma^\horizon \left( \frac{ 1 }{1 - \gamma} + \sqrt{\chi_2 (\nu \| \pi) \sigma^2 f} \frac{(\beta\gamma)^{t_0} \beta^T}{1 - \beta\gamma} \right) \leq \gamma^\horizon, \label{eq:fhn4}
	\end{align}
	where we obtain~\eqref{eq:fhn2} from~\eqref{eq:fhn1} by first collecting $\gamma^\horizon$ from the summation, summing and subtracting the term $\sum_{t = t_0}^{+\infty} \gamma^t \E_{X \sim \pi} [f(X)]$, and then applying the triangle inequality, we employ Lemma~\ref{lemma:t_stationary_inequality} to write \eqref{eq:fhn3}, and we let $t_0 \to \infty$ to tighten the bound and obtain the last inequality in~\eqref{eq:fhn4}. Let us now move to the FHC estimator. We consider a similar derivation with $s_0 \in \mathbb{N} $ and $t_0 \in \dsb{0,T}$:
	\begin{align}
	\Bias_{H_N \sim P_{\nu,\vrho}^N}&[\widehat{\eta}_{\text{FHC}} (H_N, f)] = \left|\frac{1-\gamma}{1-\gamma^T} \sum_{t=0}^{T-1} \gamma^t \E_{X \sim \nu P^t}[f(X)] - (1-\gamma) \sum_{t=0}^{+\infty} \gamma^t \E_{X \sim \nu P^t}[f(X)] \right| \\
    & = (1-\gamma) \gamma^T \left|\sum_{t=0}^{T-1}  \frac{\gamma^{t}}{1-\gamma^T} \E_{X \sim \nu P^t}[f(X)] -  \sum_{t=0}^{+\infty} \gamma^t \E_{X \sim \nu P^{t+T}}[f(X)]\right| \\
    & \le (1-\gamma) \gamma^T \Bigg(\sum_{t=0}^{t_0-1} \gamma^t + \sum_{t=t_0}^{T-1}  \frac{\gamma^{t}}{1-\gamma^T} \left| \E_{X \sim \nu P^t}[f(X)] -\E_{X \sim \pi}[f(X)] \right| \\
    & \quad + \sum_{t=0}^{s_0-1} \gamma^t + \sum_{t=s_0}^{+\infty} \gamma^t \left| \E_{X \sim \nu P^{t+T}}[f(X)] - \E_{X \sim \pi}[f(X)] \right|\Bigg) \label{line:901}\\
    & \le  (1-\gamma) \gamma^T \left( \frac{2-\gamma^{t_0}-\gamma^{s_0}}{1-\gamma} + \left( \frac{1}{1-\gamma^T} \sum_{t=t_0}^{T-1}  \gamma^t \beta^t +  \sum_{t=s_0}^{+\infty} \gamma^t \beta^{t+T}\right) \sqrt{\chi_2(\nu \| \pi) \sigma^2 f} \right)\label{line:902}\\
    & \le  (1-\gamma) \gamma^T \left( \frac{2-\gamma^{t_0}-\gamma^{s_0}}{1-\gamma} + \left( \frac{(\beta\gamma)^{t_0}(1-(\beta\gamma)^{T-t_0})}{(1-\gamma^T)(1-\beta\gamma)} + \frac{(\beta\gamma)^{s_0}\beta^{T}}{1-\beta\gamma}\right) \sqrt{\chi_2(\nu \| \pi) \sigma^2 f} \right),
\end{align}
where line~\eqref{line:901} follows from triangle inequality and line~\eqref{line:902} is obtained by applying Lemma~\ref{lemma:t_stationary_inequality}. The result is obtained by making explicit the minimization over $t_0$ and $s_0$.
\end{proof}

\biasCoroll*
\begin{proof}
	We start considering the case $\beta=0$. Since $T \ge 1$, we distinguish between the case in which the optimal value of $t_0$ is $0$ or grater than $0$. Instead, for $s_0$ it is always convenient to select $s_0=0$. If $t_0=0$, the bias bound becomes:
	\begin{align*}
		b_{0,\gamma, T}^{\text{FHC}} \rvert_{t_0=0} = (1-\gamma) \gamma^T \cdot \frac{1}{1-\gamma^T} \sqrt{\chi_2(\nu \| \pi) \sigma^2 f}.
	\end{align*}
	If instead, we select $t_0 > 0$, we get:
	\begin{align*}
		b_{0,\gamma, T}^{\text{FHC}} \rvert_{t_0>0} = (1-\gamma) \gamma^T \cdot \frac{1-\gamma^{t_0}}{1-\gamma} = \gamma^T (1-\gamma^{t_0}),
	\end{align*}
	that is minimized by selecting $t_0=1$.
	Thus, putting all together, we obtain the minimum between the two expression, whose value depend on the entity of the term $\sqrt{\chi_2(\nu \| \pi) \sigma^2 f}$. Let us move to the case $\beta=1$. Now, the bias bound becomes:
	\begin{align*}
	b_{1,\gamma, T}^{\text{FHC}} & = (1-\gamma)  \gamma^T \left( \frac{2-\gamma^{t_0}-\gamma^{s_0}}{1-\gamma} + \left( \frac{\gamma^{t_0}(1-\gamma^{T-t_0})}{(1-\gamma^T)(1-\gamma)} + \frac{\gamma^{s_0}}{1-\gamma}\right) \sqrt{\chi_2(\nu \| \pi) \sigma^2 f} \right)  \\
	& = (1-\gamma) \gamma^T \underbrace{\left( \frac{1-\gamma^{t_0}}{1-\gamma} + \frac{\gamma^{t_0}(1-\gamma^{T-t_0})}{(1-\gamma^T)(1-\gamma)}\sqrt{\chi_2(\nu \| \pi) \sigma^2 f}\right)}_{f(t_0)}\\
	& \quad  + (1-\gamma) \gamma^T \underbrace{\left( \frac{1-\gamma^{s_0}}{1-\gamma} + \frac{\gamma^{s_0}}{1-\gamma}\sqrt{\chi_2(\nu \| \pi) \sigma^2 f}\right)}_{g(s_0)}.
	\end{align*}
	We proceed in a separate way for $t_0$ and $s_0$, as they can be optimized independently. Let us start with $t_0$:
	\begin{align*}
	f(t_0)= \left( \frac{1-\gamma^{t_0}}{1-\gamma} + \frac{\gamma^{t_0}(1-\gamma^{T-t_0})}{(1-\gamma^T)(1-\gamma)}\sqrt{\chi_2(\nu \| \pi) \sigma^2 f}\right).
	\end{align*}
	It is simple to see, by renaming $x \coloneqq \gamma^{t_0}$, that $f(x)$ has no stationary points. Therefore, the optimum must be in the extreme points $t_0 \in \{0,T\}$:
	\begin{align*}
		& f(0) = \frac{1}{1-\gamma} \sqrt{\chi_2(\nu \| \pi) \sigma^2 f},\\
		& f(T) = \frac{1-\gamma^T}{1-\gamma}.
	\end{align*}
	Let us now move considering $s_0$:
	\begin{align*}
		g(s_0) =  \frac{1-\gamma^{s_0}}{1-\gamma} + \frac{\gamma^{s_0}}{1-\gamma}\sqrt{\chi_2(\nu \| \pi) \sigma^2 f}.
	\end{align*}
	Similarly to the previous case, function $g$ admits no stationary points and, thus, we consider the extreme values:
	\begin{align*}
		& g(0) = \frac{1}{1-\gamma} \sqrt{\chi_2(\nu \| \pi) \sigma^2 f}, \\
		& g(+\infty) = \frac{1}{1-\gamma}.
	\end{align*}
	Putting all together, we obtain:
	\begin{align*}
	b_{1,\gamma, T}^{\text{FHC}} & = (1-\gamma) \gamma^T \left( \min\{f(0),f(T)\} + \min\{g(0),g(+\infty)\} \right) \\
	& = \gamma^T \left(\min \left\{\sqrt{\chi_2(\nu \| \pi) \sigma^2 f}, 1-\gamma^T \right\} + \left\{\sqrt{\chi_2(\nu \| \pi) \sigma^2 f}, 1 \right\} \right) \\
	& \le 2 \gamma^T \min \left\{\sqrt{\chi_2(\nu \| \pi) \sigma^2 f}, 1 \right\},
	\end{align*}
	where the last inequality is to obtain a more interpretable expression.
\end{proof}

\paragraph{Concentration}

\concFH*

\begin{proof}
We provide a derivation that holds for both the FHN and the FHC estimators. Specifically, we consider a constant $c_{\text{FH}\star}$ with $\star \in \{\text{N,C}\}$, that is differently defined for the FHN and FHC estimators as follows:
\begin{align*}
	c_{\text{FH}\star} = \begin{cases}
							1-\gamma & \text{if } \star = \text{N} \\
\frac{1-\gamma}{1-\gamma^T} & \text{if } \star = \text{C}
							\end{cases}.
\end{align*}
Let us consider the moment--generating function for $t \in \mathbb{R}$:
	\begin{align*}
	\E\left[ \exp \left(t \left(c_{\text{FH}\star} \sum_{i=0}^{M-1} \sum_{j=0}^{T-1} \gamma^j f(X_{Ti+j}) - \pi_\gamma f\right)  \right) \right]  & = \E\left[ \prod_{i=0}^{M-1} \exp \left(t \left(c_{\text{FH}\star} \sum_{j=0}^{T-1} \gamma^j f(X_{Ti+j}) - \pi_\gamma f\right) \right) \right]\\
	& = \prod_{i=0}^{M-1} \E\left[  \exp \left(t \left(c_{\text{FH}\star} \sum_{j=0}^{T-1} \gamma^j f(X_{Ti+j}) - \pi_\gamma f\right) \right) \right] \\
	& = \E\left[  \exp \left(t \left(c_{\text{FH}\star} \sum_{j=0}^{T-1} \gamma^j f(X_{j}) - \pi_\gamma f \right) \right) \right]^M,
	\end{align*}
	where the last but one equality follows from the fact that each trajectory is independent from the others, since the reset is based on the horizon only, and the last inequality is obtained by observing that the trajectories are identically distributed. Let us now focus on one trajectory only and we highlight a bias term:
	\begin{align*}
	 \E\left[  \exp \left(t \left(c_{\text{FH}\star} \sum_{j=0}^{T-1} \gamma^j f(X_{j}) - \pi_\gamma f \right) \right) \right] & = \E\left[  \exp \left(t c_{\text{FH}\star} \sum_{j=0}^{T-1} \gamma^j (f(X_{j}) - \nu P^j f)  \right) \right] \\
	 & \qquad \times \exp \left( t \left( c_{\text{FH}\star} \sum_{j=0}^{T-1} \gamma^j  \nu P^j f - \pi_\gamma f  \right)\right) \\
	 & \le \E\left[  \exp \left(t c_{\text{FH}\star} \sum_{j=0}^{T-1} \gamma^j (f(X_{j}) - \nu P^j f)  \right) \right] \exp \left( t b_{\beta,\gamma,T}^{\text{FH}\star}\right),
	\end{align*}
	having observed that the last term corresponds to the actual bias, as bounded in Proposition~\ref{prop:biasFH}. 
	Focusing on the first term, we rename $\widetilde{f}(X_j) \coloneqq f(X_{j}) - \nu P^j f$, let $j_0 \in \dsb{0,T}$ and we apply H\"older's inequality with exponent $q \in [1,+\infty]$:
	\begin{align*}
	\E\left[ \exp \left(t c_{\text{FH}\star} \sum_{j=0}^{T-1} \gamma^j \widetilde{f}(X_j)  \right) \right] & = \E\left[ \exp \left(t c_{\text{FH}\star} \sum_{j=0}^{j_0-1} \gamma^j \widetilde{f}(X_j)  \right) \exp \left(t c_{\text{FH}\star}  \sum_{j=j_0}^{T-1} \gamma^j \widetilde{f}(X_j)  \right)\right]\\
	& \le \underbrace{\E\left[ \exp \left(t c_{\text{FH}\star} q \sum_{j=0}^{j_0-1} \gamma^j \widetilde{f}(X_j)  \right) \right]^{\frac{1}{q}}}_{\text{(a)}} \underbrace{\E\left[ \exp \left( \frac{tc_{\text{FH}\star}q}{q-1} \sum_{j=j_0+1}^{T-1} \gamma^j \widetilde{f}(X_j)  \right) \right]^{\frac{q-1}{q}}}_{\text{(b)}}.
	\end{align*}
	Let us focus on term (a), we look at the quantity $c_{\text{FH}\star}  q \sum_{j=0}^{j_0-1} \gamma^j \widetilde{f}(X_j) $ as a unique random variable whose range is $4 c_{\text{FH}\star}  q (1-\gamma^{j_0})/(1-\gamma)$ as $|\widetilde{f}(X_j)| \le 1$. Thus, by H\"oeffding's lemma, having observed that the terms $ \widetilde{f}(X_j)$ are zero-mean random variables:
	\begin{align*}
		\text{(a)} = \E\left[ \exp \left(t c_{\text{FH}\star} q \sum_{j=0}^{j_0-1} \gamma^j \widetilde{f}(X_j)  \right) \right]^{\frac{1}{q}} \le \exp \left( \frac{2t^2 c_{\text{FH}\star}^2 q (1-\gamma^{j_0})^2}{(1-\gamma)^2} \right).
	\end{align*}
	Let us move to term (b). Here, we need to highlight a further bias term:
	\begin{align*}
		\text{(b)} & = \E\left[ \exp \left( \frac{tc_{\text{FH}\star}q}{q-1} \sum_{j=j_0+1}^{T-1} \gamma^j \widetilde{f}(X_j)  \right) \right]^{\frac{q-1}{q}} \\
		& = \E\left[ \exp \left( \frac{tc_{\text{FH}\star}q}{q-1} \sum_{j=j_0+1}^{T-1} \gamma^j (f(X_j) - \nu P^j f)  \right) \right]^{\frac{q-1}{q}} \\
		& = \underbrace{\E\left[ \exp \left( \frac{tc_{\text{FH}\star}q}{q-1} \sum_{j=j_0+1}^{T-1} \gamma^j (f(X_j) - \pi f)  \right) \right]^{\frac{q-1}{q}}}_{\text{(c)}} \exp \left( tc_{\text{FH}\star} \sum_{j=j_0+1}^{T-1} \gamma^j (\pi f - \nu P^j f)  \right).
	\end{align*}
	Now, we focus on term (c) and proceed with a change of measure followed by an application of H\"older's inequality with exponent $r \in [1,+\infty]$:
	\begin{align*}
	\text{(c)}  & = \E\left[ \exp \left( \frac{tc_{\text{FH}\star}q}{q-1} \sum_{j=j_0+1}^{T-1} \gamma^j (f(X_j) - \pi f)  \right) \right]^{\frac{q-1}{q}}  \\
	& = \E_{\pi}\left[ \frac{\nu P^{j_0}(X_{j_0})}{\pi(X_{j_0})} \exp \left( \frac{tc_{\text{FH}\star}q}{q-1} \sum_{j=j_0+1}^{T-1} \gamma^j (f(X_j) - \pi f)  \right) \right]^{\frac{q-1}{q}} \\
	&  \le\E_{\pi}\left[ \left( \frac{\nu P^{j_0}(X_{j_0})}{\pi(X_{j_0})} \right)^r \right]^{\frac{q-1}{rq}} \underbrace{\E_{\pi}\left[ \exp \left( \frac{tc_{\text{FH}\star}qr}{(q-1)(r-1)} \sum_{j=j_0+1}^{T-1} \gamma^j (f(X_j) - \pi f)  \right) \right]^{\frac{(q-1)(r-1)}{qr}}}_{\text{(d)}}.
	\end{align*}
	Then, we consider term (e) and apply Theorem 1 of~\cite{fan2021hoeffding} to bound the moment generating function, recalling that $1-\beta$ is the absolute spectral gap:
	\begin{align*}
	\text{(d)} & = \E_{\pi}\left[ \exp \left( \frac{t c_{\text{FH}\star}  qr}{(q-1)(r-1)} \sum_{j=j_0+1}^{T-1} \gamma^j (f(X_j) - \pi f)  \right) \right]^{\frac{(q-1)(r-1)}{qr}} \\
	& \le \exp \left( \frac{2 t^2 c_{\text{FH}\star}^2 qr}{(q-1)(r-1)} \cdot \frac{1+\beta}{1-\beta} \cdot \frac{\gamma^{2j_0} - \gamma^{2T}}{1-\gamma^2} \right).
	\end{align*}
Concerning the second bias term, we can provide a bound by exploiting Lemma~\ref{lemma:t_stationary_inequality}:
\begin{align*}
\sum_{j=j_0+1}^{T-1} \gamma^j (\pi f - \nu P^j f) \le \sum_{j=j_0+1}^{T-1}  (\beta\gamma)^j \sqrt{\chi_2(\nu\|\pi) \sigma^2 f} = \frac{(\beta\gamma)^{j_0} - (\beta\gamma)^T}{1-\beta\gamma} \sqrt{\chi_2(\nu\|\pi) \sigma^2 f} \coloneqq c_{\beta,\gamma}(j_0).
\end{align*}

Putting all together and by minimizing over $t$, we have, for $\star \in \{\text{N,C}\}$:
\begin{align*}
	\Pr_{H_N \sim P_{\nu,\vrho}^{N}} & \left( \left| \widehat{\eta}_{\text{FH}\star}(H_N,f) - \pi_\gamma f \right|  > \epsilon\right) \le 2 \left\| \frac{\nu P^{j_0}}{\pi} \right\|_{r,\pi}^{\frac{q-1}{q}}\\
	& \quad \times \min_{t \ge 0} \exp \Bigg(- t\left( \epsilon - b_{\beta,\gamma,T}^{\text{FH}\star} - c_{\text{FH}\star} c_{\beta,\gamma}(j_0)\right)\\
	& \qquad + 2t^2c_{\text{FH}\star}^2 q \left( \frac{(1-\gamma^{j_0})^2}{(1-\gamma)^2} + \frac{r}{(q-1)(r-1)} \cdot \frac{1+\beta}{1-\beta} \cdot \frac{\gamma^{2j_0}-\gamma^{2T}}{1-\gamma^2} \Bigg)
	 \right)^M\\
	 & =  2 \left\| \frac{\nu P^{j_0}}{\pi} \right\|_{r,\pi}^{\frac{q-1}{q}} \exp\left( - \frac{\left( \epsilon  - b_{\beta,\gamma,T}^{\text{FH}\star} - c_{\text{FH}\star} c_{\beta,\gamma}(j_0)\right)^2M}{8c_{\text{FH}\star}^2 q \left( \frac{(1-\gamma^{j_0})^2}{(1-\gamma)^2} + \frac{r}{(q-1)(r-1)} \cdot \frac{1+\beta}{1-\beta} \cdot \frac{\gamma^{2j_0}-\gamma^{2T}}{1-\gamma^2}\right)}	 \right).
\end{align*}
By solving for $\epsilon$, and minimizing over the free parameters $q$, $r$, and $j_0$, we obtain that with probability at least $1-\delta$ it holds that:
\begin{align*}
\left| \widehat{\eta}_{\text{FH}\star}(H_N,f) - \pi_\gamma f \right| & \le b_{\beta,\gamma,T}^{\text{FH}\star} +c_{\text{FH}\star} \min_{j_0 \in \dsb{0,T},q,r \ge 1}   \Bigg\{ c_{\beta,\gamma}(j_0) \\
& \quad +  \sqrt{\frac{8 q}{M} \left( \frac{(1-\gamma^{j_0})^2}{(1-\gamma)^2} + \frac{r}{(q-1)(r-1)} \cdot \frac{1+\beta}{1-\beta} \cdot \frac{\gamma^{2j_0}-\gamma^{2T}}{1-\gamma^2}\right)} \Bigg\}\\ 
& \quad \times \sqrt{ \left( \log \frac{2}{\delta} + \log  \left\| \frac{\nu P^{j_0}}{\pi} \right\|_{r,\pi}^{\frac{q-1}{q}} \right)}.
\end{align*}
To obtain the theorem statement, we set $q=r=2$, to get:
\begin{align*}
\left| \widehat{\eta}_{\text{FH}\star}(H_N,f) - \pi_\gamma f \right| & \le b_{\beta,\gamma,T}^{\text{FH}\star} + c_{\text{FH}\star} \min_{j_0 \in \dsb{0,T}}   c_{\beta,\gamma}(j_0) +  \sqrt{\frac{16}{M} \left( \frac{(1-\gamma^{j_0})^2}{(1-\gamma)^2} + 2 \cdot \frac{1+\beta}{1-\beta} \cdot \frac{\gamma^{2j_0}-\gamma^{2T}}{1-\gamma^2}\right)}\\ 
& \quad \times \sqrt{ \left( \log \frac{2}{\delta} + \frac{1}{4} \log  \left\| \frac{\nu P^{j_0}}{\pi} \right\|_{2,\pi}^2 \right)} \\
& \le b_{\beta,\gamma,T}^{\text{FH}\star} + c_{\text{FH}\star} \min_{j_0 \in \dsb{0,T}}  c_{\beta,\gamma}(j_0) +  \sqrt{ \left( \frac{(1-\gamma^{j_0})^2}{(1-\gamma)^2} + \frac{1+\beta}{1-\beta} \cdot \frac{\gamma^{2j_0}-\gamma^{2T}}{1-\gamma^2}\right)}\\ 
& \quad \times \sqrt{ \frac{32}{M}\left( \log \frac{2}{\delta} + \frac{1}{4} \log \left(\chi_2(\nu P^{j_0}\| \pi) + 1\right) \right)},
\end{align*}
the statement is obtained by observing that $M = T/N$.
	\end{proof}

\concFHcoroll*
\begin{proof}
	We start with the case $\beta=1$. From Theorem~\ref{thr:concFH}, we immediately observe that we need to select $j_0=T$, otherwise the concentration bound degenerates to infinity. Moreover, we make use of the bias bounds of Corollary~\ref{coroll:FH}. Thus, ignoring the term $\chi_2(\nu\|\pi)\sigma^2f$, we obtain the $O$ expression provided in the statement of the corollary.
	Concerning the case $\beta=0$, instead, we need some additional care. First of all, we observe that $c_{0,\gamma}(j_0) = 0$ for every $j_0 > 0$, and $c_{0,\gamma}(0) = \sqrt{\chi_2(\nu\|\pi)\sigma^2f}$. Thus, we can ignore this term. Then, it is simple to verify that the expression $d_{0,\gamma}(j_0)$, ignoring the dependence on $\chi_2(\nu\|\pi)\sigma^2f$ again, is given by:
	\begin{align*}
		d_{0,\gamma}(j_0) = O \left( \sqrt{\frac{(1-\gamma^{j_0})^2}{(1-\gamma)^2} + \frac{\gamma^{2j_0}- \gamma^{2T}}{1-\gamma^2}} \sqrt{\frac{T \log \frac{2}{\delta}}{N}}\right).
	\end{align*}
	By vanishing the derivative, we obtain the value of $j_0$ that is minimizing the expression, \ie $j_0 = \frac{\log ((1+\gamma)/2)}{\log \gamma}$. This quantity is in the interval $[0,1/2]$ varying $\gamma \in [0,1]$. Consequently, as $j_0$ must be integer, we select $j_0 = 0$, to get the expression shown in the corollary statement.
\end{proof}

\subsubsection{Adaptive-Horizon Estimation Algorithms}

\paragraph{Computational Analysis}

\ahEstimatorComp*
\begin{proof}
	To characterize the number of trajectories, we consider the sampling process, in which, at every step $t$, we sample independently a Bernoulli random variable to decide whether to reset:
	\begin{align*}
		Y_t \sim \rho_t^{\text{AHR}}(H_t,X_t) = \mathrm{Ber}(1-\gamma).
	\end{align*}
	We have already observed that the number of trajectories can be computed as $M = 1+\sum_{t=0}^{N-1} Y_i$. Consequently, we have that $M-1$ is the sum of $N$ independent Bernoulli random variables, being a binomial random variable $\mathrm{Bin}(N-1,1-\gamma)$. From the properties of the binomial random variable, we have that $\E[M-1] = (N-1)(1-\gamma)$.
	
	To analyze the time complexity, we need to characterize the distribution of the maximum length among the trajectories, \ie $T_{\max} = \max_{i \in \dsb{M}} T_i$. Each $T_i$ can be looked as derived from a geometric distribution as $T_i-1 \sim \mathrm{Geo}(1-\gamma)$. Unfortunately, these random variables are dependent (but identically distributed) since the process stops as soon as the have run out of budget. To this end, we will proceed as follows, being $k \in \mathbb{N}$:
	\begin{align*}
		\Pr\left( \max_{i \in \dsb{M}} T_{i-1} > k\right) & = \E \left[ \indic\left\{\max_{i \in \dsb{M}} T_{i-1} > k \right\} \right] \\
		& = \E \left[ \sum_{m=1}^N \indic\left\{M=m \right\} \indic\left\{\max_{i \in \dsb{M}} T_{i-1} > k \right\} \right]  \\
		& \le  \sum_{m=1}^N  \E \left[\indic\left\{\max_{i \in \dsb{m}} T_{i-1} > k \right\} \right] \\
		& =  \sum_{m=1}^N \Pr \left( \max_{i \in \dsb{m}} T_{i-1} > k \right).
	\end{align*}
	Now, we consider one term at a time and perform a union bound:
	\begin{align*}
	\Pr \left( \max_{i \in \dsb{m}} T_{i-1} > k \right) & = \Pr \left( \bigvee_{i \in \dsb{m}} T_{i-1} > k \right) \\
	& \le \sum_{i \in \dsb{m}} \Pr\left(T_{i-1} > k  \right) \\
	& = m \Pr\left(T_{0} > k  \right),
	\end{align*}
	where the last equality follows from the fact that the random variables $T_i$ are identically distributed. Since $T_0 - 1 $ is a geometric distributions, we have that $\Pr\left(T_{0} \ge k  \right) = \gamma^{k-1}$. Thus, putting all together, we obtain:
	\begin{align*}
	\sum_{m=1}^N \Pr \left( \max_{i \in \dsb{m}} T_{i-1} > k \right) = \sum_{m=1}^N m \gamma^{k-1} = \frac{N(N+1)}{2} \gamma^{k-1}.
	\end{align*}
	Solving to obtain $k$, we have that with probability at least $1-\delta$ it holds that:
	\begin{align*}
	k \le 1+\frac{\log\frac{N(N+1)}{2\delta}}{\log \frac{1}{\gamma}} \le 1+\frac{\log \frac{2N^2}{2\delta}}{1-\gamma},
	\end{align*}
	having observed that $N+1 \le N$ and $\frac{1}{\log\frac{1}{\gamma}} \le \frac{1}{1-\gamma}$. By taking the minimum with the number of samples $N$, we get the result.
\end{proof}

\paragraph{Statistical Analysis}

\osEstimatorConc*

\begin{proof}
	Suppose that the number of trajectories $M-1=m-1$ is fixed. In this case, we can apply H\"oeffding's inequality to the estimator:\footnote{Note that conditioning to $M=m$ is allowed as the decision to reset is independent on the values of $f(X_t)$ but depends on an independent trial $Y_t \sim \mathrm{Ber}(1-\gamma)$ at each step.}
	\begin{align*}
		\Prob_{H_N \sim P_{\nu,{\vrho}^{\text{AHR}}}^N} \left( \left| \widehat{\eta}_{\text{OS}}(H_N,f) - \pi_\gamma f \right| > \epsilon | M = m\right) \le 2 \exp \left( - \frac{\epsilon^2 (m-1)}{2} \right).
	\end{align*}
	Now, we take the expectation \wrt to the distribution of $M$ that is a binomial distribution:
	\begin{align*}
		\Prob_{H_N \sim P_{\nu,{\vrho}^{\text{AHR}}}^N} \left( \left| \widehat{\eta}_{\text{OS}}(H_N,f) - \pi_\gamma f \right| > \epsilon\right) & = \E_{m \sim \mathrm{Bin}(N,1-\gamma)} \left[ \Prob_{H_N \sim P_{\nu,{\vrho}^{\text{AHR}}}^N} \left( \left| \widehat{\eta}_{\text{OS}}(H_N,f) - \pi_\gamma f \right| > \epsilon | M = m\right) \right] \\
		& \le \E_{m \sim \mathrm{Bin}(N,1-\gamma)} \left[ 2 \exp \left( - \frac{\epsilon^2 m}{2} \right) \right] \\
		& = 2 \sum_{m=0}^N {N \choose m} (1-\gamma)^m \gamma^{N-m} \exp \left( - \frac{\epsilon^2 m}{2} \right) \\
		& = 2 \left( \gamma + (1-\gamma) \exp\left( -\frac{\epsilon^2}{2}\right)\right)^N.
	\end{align*}
	 We now provide a looser but more interpretable bound. To this end, we consider the derivation, holding for $\epsilon\in [0,1]$ (since $f(x) \in [0,1]$ for all $x \in \Xs$):
	\begin{align*}
	\frac{\gamma + (1-\gamma) \exp\left( -\frac{\epsilon^2}{2}\right)}{\exp\left( - \frac{\epsilon^2(1-\gamma)}{2} \right)} & = \gamma \exp \left(\frac{\epsilon^2(1-\gamma)}{2} \right) + (1-\gamma) \exp \left( -\frac{\epsilon^2 \gamma}{2} \right) \\
	& \begin{cases}
		\le 1+e \\
		\ge \gamma +(1-\gamma) \exp(-\gamma) \ge \exp(-\gamma) \ge e^{-1}
	\end{cases}
	\end{align*}
	Thus, we have that with probability at least $1-\delta$ it holds that:
	\begin{align*}
	\left| \widehat{\eta}_{\text{OS}}(H_N,f) - \pi_\gamma f \right| \le \sqrt{\frac{2 \log \frac{2(1+e)}{\delta}}{N(1-\gamma)}}.
	\end{align*}
	The result is obtained by observing that $2(1+e) < 8$.
\end{proof}

\asEstimatorConc*

\begin{proof}
	We start by working on the moment-generating function. Let $t \in \Reals$. Let us consider $i_0 \in \dsb{0,N}$:
	\begin{align*}
		\E\left[ \exp \left( t \sum_{i=0}^{N-1} (f(X_i) - \pi_\gamma f ) \right)\right] & = \E\left[ \exp \left( t \sum_{i=0}^{i_0-1} (f(X_i) - \pi_\gamma f ) \right) \exp \left( t \sum_{i=i_0}^{N-1} (f(X_i) - \pi_\gamma f )\right) \right] \\
		& \le  \exp \left( i_0 t\right)  \E\left[  \exp \left( t \sum_{i=i_0}^{N-1} (f(X_i) - \pi_\gamma f ) \right)\right],
	\end{align*}
	where we exploited the inequality $|f(X_i) - \pi_\gamma f| \le 1$. We now move to bound the second term, by exploiting a change of measure argument and H\"older's inequality with $q \in [1,+\infty]$:
	\begin{align*}
\E\left[  \exp \left( t \sum_{i=i_0}^{N-1} (f(X_i) - \pi_\gamma f ) \right)\right] & = \E_{\pi_\gamma} \left[ \frac{\nu P^{i_0}_\gamma({X_{i_0}})}{\pi_\gamma({X_{i_0}})}  \exp \left( t \sum_{i=i_0}^{N-1} (f(X_i) - \pi_\gamma f ) \right)\right] \\
& \le  \E_{\pi_\gamma} \left[ \left(\frac{\nu P^{i_0}_\gamma({X_{i_0}})}{\pi_\gamma({X_{i_0}})} \right)^q\right]^{\frac{1}{q}} \E_{\pi_\gamma} \left[  \exp \left(\frac{ t q}{q-1}\sum_{i=i_0}^{N-1} (f(X_i) - \pi_\gamma f ) \right)\right]^{\frac{q-1}{q}} \\
& = \left\|\frac{\nu P^{i_0}_\gamma}{\pi_\gamma} \right\|_{\pi_\gamma,q} \E_{\pi_\gamma} \left[  \exp \left(\frac{ t q}{q-1}\sum_{i=i_0}^{N-1} (f(X_i) - \pi_\gamma f ) \right)\right]^{\frac{q-1}{q}}
	\end{align*}
%	To bound the first term, we exploit Theorem 12 of~\cite{fan2021hoeffding}, to get:
%	\begin{align*}
%		\left\|\frac{\nu P^{i_0}_\gamma}{\pi_\gamma} \right\|_q \le 1 + 2^{2 \max\{\frac{1}{q}, 1-\frac{1}{q}\}} (\beta\gamma)^{2 i_0 \min\{\frac{1}{q}, 1-\frac{1}{q}\}} \left\|\frac{\nu}{\pi_\gamma} - 1 \right\|_q.
%	\end{align*}
%	Specifically, for $q=2$, we get:
%	\begin{align*}
%	\left\|\frac{\nu P^{i_0}_\gamma}{\pi_\gamma} \right\|_2 \le 1 + (\beta\gamma)^{i_0} \chi_2\left( \nu \| \pi_\gamma \right).
%	\end{align*}
Now, we exploit Lemma~\ref{lemma:discounted_kernel} to derive that the absolute spectral gap of $P_\gamma$ is $1-\beta\gamma$, being $\beta\gamma$ the second eigenvalue of operator $P_\gamma$. To bound the expectation in the previous equation, we exploit Theorem 1 of~\cite{fan2021hoeffding}:
	\begin{align*}
	\E_{ \pi_\gamma} \left[  \exp \left(\frac{ t q}{q-1}\sum_{i=i_0}^{N-1} (f(X_i) - \pi_\gamma f ) \right)\right]^{\frac{q-1}{q}} \le \exp \left( \frac{2t^2q}{q-1} \cdot (N-i_0) \cdot \frac{1+\beta\gamma}{1-\beta\gamma} \right).
	\end{align*}
	We can now proceed to bound the probability, by minimizing over $t \ge 0$:
	\begin{align*}
		\Pr_{H_N \sim P_{\nu,\vrho}^N} \left( |\widehat{\eta}_{\text{AS}}(H_N,f) - \pi_\gamma f| > \epsilon \right) & \le 	2\left\|\frac{\nu P^{i_0}_\gamma}{\pi_\gamma} \right\|_{\pi_\gamma, q}
		\min_{t \ge 0} \exp \left( - t (\epsilon N -i_0) + \frac{2t^2q}{q-1} \cdot (N-i_0) \cdot \frac{1+\beta\gamma}{1-\beta\gamma} \right) \\
		& = 	2\left\|\frac{\nu P^{i_0}_\gamma}{\pi_\gamma} \right\|_{\pi_\gamma, q}\exp \left( - \frac{(\epsilon N - i_0)^2}{\frac{2q}{q-1} \cdot (N-i_0) \cdot \frac{1+\beta\gamma}{1-\beta\gamma}} \right) \\
		& = 	2\left\|\frac{\nu P^{i_0}_\gamma}{\pi_\gamma} \right\|_{\pi_\gamma,q} \exp \left( - (\epsilon N - i_0)^2\cdot \frac{q-1}{2q (N-i_0)}  \cdot \frac{1-\beta\gamma}{1+\beta\gamma}\right). 
	\end{align*}
%	By taking $q = 2$, we get:
%	\begin{align*}
%		\Pr_{H_N \sim } \left( \widehat{\eta}_{\text{AS}}(H_N,f) - \pi_\gamma f > \epsilon \right) \le \left( 1 + (\beta\gamma)^{i_0} \chi_2\left( \nu \| \pi_\gamma \right) \right)  \left( - (\epsilon N - i_0)^2\cdot \frac{q-1}{2q (N-i_0)}  \cdot \frac{1-\beta\gamma}{1+\beta\gamma}\right).
%	\end{align*}
By solving for $\epsilon$, and minimizing over the free parameters $q$ and $i_0$, we obtain that with probability at least $1-\delta$ it holds that:
	\begin{align*}
	\left| \widehat{\eta}_{\text{AS}}(H_N,f) - \pi_\gamma f\right| & \le \min_{i_0 \in \dsb{0,N}, q \ge 1} \frac{i_0}{N} + \sqrt{\frac{2q(1-i_0/N)}{N(q-1)} \cdot \frac{1+\beta\gamma}{1-\beta\gamma} \left(\log  \frac{2}{\delta} + \log \left\|\frac{\nu P^{i_0}_\gamma}{\pi_\gamma} \right\|_{\pi_\gamma,q} \right)}.
\end{align*}
Since the optimization is non-trivial, the result shown in the statement of the theorem is obtained by setting $q=2$ and $i_0=0$, observing that $\left\|\frac{\nu }{\pi_\gamma} \right\|_2 ^2 = \chi_2\left( \nu \| \pi_\gamma \right)+1$ and bounding $1+\beta\gamma \le 2$.
\end{proof}

\begin{prop}[AS Estimator - Bias] Let $H_N \sim P^N_{\nu,\textcolor{black}{\vrho}}$ with the reset policy $\rho_t^{\text{AHR}}(H_t,X_t) = \mathrm{Ber}(1-\gamma)$, and let $f : \Xs \to [0, 1]$. Then, it holds that:
\begin{align*}
	\Bias_{H_N \sim P_{\nu,\vrho}^N}[\widehat{\eta}_{\text{AS}} (H_N, f)] \le \frac{1-(\beta\gamma)^N}{N(1-\beta\gamma)} \sqrt{\chi_2(\nu\|\pi_\gamma) \sigma^2_\gamma f}.
\end{align*}
\end{prop}

\begin{proof}
	Let us consider the following derivation:
	\begin{align}
		   \Bias_{H_N \sim P_{\nu,\vrho}^N}[\widehat{\eta}_{\text{AS}} (H_N, f)]&  = \left|\frac{1}{N} \sum_{t=1}^{N} \E_{H_N \sim P_{\nu,\vrho}^N}[\widehat{\eta}_{\text{AS}} (H_N, f)] - \pi_\gamma f \right| \notag \\
		   & =  \left|\frac{1}{N} \sum_{t=1}^{N} \E_{X \sim \nu P_\gamma^t} \E[f(X)] - \pi_\gamma f \right|\notag \\ 
		   & \le \sqrt{\chi_2(\nu\|\pi_\gamma) \sigma^2_\gamma f} \frac{1}{N} \sum_{t=0}^{N-1} (\beta\gamma)^t \label{eq:-501}\\
    & =  \sqrt{\chi_2(\nu\|\pi_\gamma) \sigma^2_\gamma f}\frac{1-(\beta\gamma)^N}{N(1-\beta\gamma)}. \notag
	\end{align}
	where line~\eqref{eq:-501} follows from Lemma~\ref{lemma:t_stationary_inequality} an d recalling that the absolute spectral gap of $P_\gamma$ is $\beta\gamma$. 
\end{proof}

\subsubsection{About the Optimal Horizon $T$}\label{sec:magic3}
In this appendix, we elaborate on the choice of the horizon $T$ for the FH estimators in a way that it is independent on the mixing properties of the Markov chain. To this end, we consider the simplified expressions of the concentration bounds of Corollary~\ref{coroll:FH}.
Let us define:
\begin{align*}
	T^*_\gamma \coloneqq \frac{\log \sqrt{N}}{\log \frac{1}{\gamma}} .
\end{align*}

\paragraph{Finite Horizon Corrected Estimator}
For the FHC estimator we show that the choice of $T = T^*_\gamma$ makes the concentration rate nearly minimax optimal. Indeed, for $\beta = 0$, we have:
\begin{align*}
	\left| \widehat{\eta}_{\text{FHC}} (H_N, f) - \pi_\gamma f \right|  & \le O \left( (1-\gamma)\gamma^{T^*_\gamma} + \sqrt{\frac{{T^*_\gamma} (1-\gamma) \log \frac{2}{\delta}}{N(1-\gamma^{T^*_\gamma})}} \right) \\
	& = O \left(\frac{1-\gamma}{\sqrt{N}} + \sqrt{\frac{\log \sqrt{N} (1-\gamma) \log \frac{2}{\delta}}{N \log \frac{1}{\gamma} \left(1 - \frac{1}{\sqrt{N}}\right) }} \right) \\
	& \le O \left( \sqrt{\frac{\log N\log \frac{2}{\delta}}{N }} \right) \\
	& = \widetilde{O} \left( \frac{1}{\sqrt{N}}\right).
\end{align*}
having observed that $(1-\gamma) / \log(1/\gamma) \le 1$ and whenever $N \ge 2$. This concentration rate is indeed matching, in $\widetilde{O}$ sense, the minimax rate. We consider now $\beta = 1$. A similar derivation applies:
\begin{align*}
\left| \widehat{\eta}_{\text{FHC}} (H_N, f) - \pi_\gamma f \right|  & \le O \left( \gamma^{T^*_\gamma} + \sqrt{\frac{{T^*_\gamma}\log \frac{2}{\delta}}{N}} \right) \\
& \le O \left(\frac{1}{\sqrt{N}} + \sqrt{\frac{\log \sqrt{N} \log \frac{2}{\delta}}{N \log \frac{1}{\gamma}}} \right) \\
& \le O \left( \sqrt{\frac{\log N \log \frac{2}{\delta}}{2N(1-\gamma)}} \right) \\
& = \widetilde{O} \left(\frac{1}{\sqrt{N(1-\gamma)}} \right),
\end{align*}
having bounded $1/\log(1/\gamma) \le 1/(1-\gamma)$. This rate matches as well, in the $\widetilde{O}$ sense, the minimax concentration rate.

\paragraph{Finite Horizon Non-Corrected Estimator} The choice of $T = T^*_\gamma$ happens to make also the concentration rate of the FHN estimator nearly minimax optimal (in the $\widetilde{O}$ sense) for $\beta = 1$. Indeed,  we have:
\begin{align*}
	 \left| \widehat{\eta}_{\text{FHN}} (H_N, f) - \pi_\gamma f \right|  & \le O \left( \gamma^{T^*_\gamma} + \sqrt{\frac{{T^*_\gamma} \log \frac{2}{\delta}}{N}} \right) \\
	 & \le O \left( \frac{1}{\sqrt{N}} +  \sqrt{\frac{\log \sqrt{N} \log \frac{2}{\delta}}{N\log \frac{1}{\gamma}}} \right) \\
	 & \le O \left(\sqrt{\frac{\log N \log \frac{2}{\delta}}{N(1-\gamma)}} \right) \\
	 & = \widetilde{O} \left( \frac{1}{\sqrt{N(1-\gamma)}} \right).
\end{align*}
{
For the case $\beta = 0$, we have:
\begin{align*}
 \left| \widehat{\eta}_{\text{FHN}} (H_N, f) - \pi_\gamma f \right|  & \le O \left( \gamma^{T^*_\gamma} + \sqrt{\frac{{T^*_\gamma (1-\gamma)(1-\gamma^{T^*_\gamma})} \log \frac{2}{\delta}}{N}} \right) \\
	 & \le O \left( \frac{1}{\sqrt{N}} +  \sqrt{\frac{(1-\gamma) \log \sqrt{N} \log \frac{2}{\delta}}{N\log \frac{1}{\gamma}}} \right) \\
	 & = \widetilde{O} \left( \frac{1}{\sqrt{N}} \right).
\end{align*}
}
%
%Unfortunately, for the case $\beta = 0$, the optimal value of $T$ is unable to make the concentration bound match the minimax rate. To this end, let us consider the following expression obtained from Corollary~\ref{coroll:FH} by applying $\widetilde{O}$:
%\begin{align*}
%	\gamma^T + (1-\gamma^T) \sqrt{\frac{T}{N}}.% = \gamma^T \left(1-\sqrt{\frac{T}{N}}\right) + \sqrt{\frac{T}{N}} \ge \frac{\gamma^T}{2} + \sqrt{\frac{T}{N}},
%\end{align*}

%In order to make this term minimax optimal, one must select $T \le \widetilde{O}(1)$ and $\gamma^T \le \widetilde{O}(1/{\sqrt{N})$, that implies $T \ge \widetilde{O}(1/(\log(1/\gamma))$. Which is clearly impossibile.
%$T \ge \frac{\log a \sqrt{N}}{\log \frac{1}{\gamma}}$
%$T \le b$
%\am{Completare: Bisogna far vedere che il minimo in $T$ dell'espressione sopra è $O(1/\sqrt{(1-\gamma)N})$.}

{
\subsubsection{About Minimax Optimality of FH Estimators for generic $\beta \in (0,1)$}\label{sec:magicSec}
In this appendix, we provide further elaboration about the possible minimax optimality of the FH estimators for a generic value of $\beta \in (0,1)$. Specifically, we show that, according to our analysis, there exists a regime of large values of $\beta$, \ie $\beta \in (\overline{\beta},1)$, for which our bound of Theorem~\ref{thr:concFH} cannot match the minimax lower bound. To this end, we consider a simplified version of the bound of Theorem~\ref{thr:concFH}, that disregards the bias term $c_{\beta,\gamma}(j_0)$ and simply focus on the term:
\begin{align*}
	f(j_0) \coloneqq {\frac{(1-\gamma^{j_0})^2}{(1-\gamma)^2} + \frac{1+\beta}{1-\beta}\cdot \frac{\gamma^{2j_0} - \gamma^{2T}}{1-\gamma^2}}.
\end{align*}
Let us minimize this term over $j_0 \in \dsb{0,T}$. We can proceed by vanishing the derivative of $f$ in $j_0$. It is simple to understand (\eg by performing the substitution $x = \gamma^{j_0}$, obtaining a quadratic function in $x$) that the only stationary point $j^*_0$ is the global minimum. However, it might be the case that such point is larger than $T$. In this case, we need to clip $j_0^*$ to $T$. Thus, we have: 
\begin{align*}
	j_0^* = \begin{cases}
				\frac{\log \frac{(1-\beta)(1+\gamma)}{2(1-\beta\gamma)}}{\log \gamma} & \text{if } \beta \le \frac{1+\gamma-2\gamma^T}{1+\gamma-2\gamma^{T+1}} \\
				 T &  \text{otherwise}
			\end{cases}.
\end{align*}
We let $\overline{\beta} = \frac{1+\gamma-2\gamma^T}{1+\gamma-2\gamma^{T+1}}$ and show that for $\beta \in (\overline{\beta},1)$ the FH estimators are not minimax optimal. To show this we consider first the FHN estimator, that leads to a bound of the form:
\begin{align*}
\gamma^T + (1-\gamma^T) \sqrt{\frac{T}{N}}.
\end{align*}
First of all, we observe that the value of $T$ minimizing the previous expression must be sublinear in $N$, because the second addendum will not shrink as $N \rightarrow + \infty$ otherwise. Thus, w.l.o.g., we consider the case $T/N \le \frac{1}{4}$. We have:
\begin{align*}
 \gamma^T + (1-\gamma^T) \sqrt{\frac{T}{N}} \ge \frac{\gamma^T}{2} + \frac{1}{2}\sqrt{\frac{T}{N}}.
\end{align*}
By applying Lemma~\ref{lemma:boh}, we obtain that the minimum value of this function over $T$, for $\gamma \ge 0.3$, is given by:
\begin{align*}
	\frac{1}{\sqrt{24N(1-\gamma)}}.
\end{align*} 
Thus, we conclude that FHN cannot match the minimax lower bound. A similar derivation can be set for the FHC estimator which leads to the bound ($0.3 \le \gamma < 1$):
\begin{align*}
 (1-\gamma)\gamma^T + \sqrt{\frac{T}{N}} \ge \frac{1}{\sqrt{6N(1-\gamma)}}.
\end{align*}
%
%
%Let us now consider the first case in which $\beta \le \frac{1+\gamma-2\gamma^T}{1+\gamma-2\gamma^{T+1}}$, we have:
%\begin{align*}
%	f(j_0^*) = \frac{(1+\beta) \left(\gamma +\beta  \left(-\gamma +2 \gamma ^{2 T+1}-1\right)-2 \gamma ^{2 T}+1\right)}{2 (1-\beta ) \left(1-\gamma ^2\right) (1-\beta  \gamma )} \le \frac{(1+\gamma)(1-\gamma^T)^2}{(1-\gamma)(1+\gamma-2\gamma^{1+T})(1-\beta\gamma)}.
%\end{align*}
%Instead, for the second case:
%\begin{align*}
%	f(j_0^*) = \left( \frac{1-\gamma^T}{1-\gamma}\right)^2.
%\end{align*}
%Since we must select the same value of $T$ for all $\beta$, being a $\beta$-independent choice, we take
}

\subsection{Technical Lemmas}

\begin{lemma}\label{lemma:t_stationary_inequality}
	Let $\pi$ be an invariant measure of a Markov chain $P$ with spectral gap $1 - \beta \in (0, 1]$ and initial-state distribution $\nu$. For any bounded measurable function $f \in \mathscr{B} (\Xs)$, it holds that:
	\begin{equation*}
		\left| \nu P^t f - \pi f \right| \leq \sqrt{\chi_2 (\nu\|\pi)} \beta^t \sqrt{\sigma^2 f},
	\end{equation*}
	where $\sigma^2 f = \pi (f - \mathbf{1} \pi f)^2$ is the variance of $f$ under $\pi$.
\end{lemma}

\begin{proof}
	We exploit the fact that $\nu \Pi = \pi \Pi = \pi$ and that $\pi P^t = \pi$ to derive the following identity for any $c \in \Reals$:
    \begin{equation*}
         \nu P^t f - \pi f = \nu P^t f - \nu \Pi f = (\nu - \pi) (P^t - \Pi) f = (\nu - \pi) (P^t - \Pi) (f - c).
    \end{equation*}
    Then, we start from the right-hand side of the identity to write
    \begin{align*}
        \int_{\Xs} (&\nu(\de x) - \pi (\de x)) \int_{\Xs} (P^t(\de y|x) - \pi(\de y)) (f(y) - c) \\  
        &=  \int_{\Xs} \pi(\de x) \frac{\nu(\de x) - \pi (\de x)}{\pi(\de x)} \int_{\Xs} (P^t(\de y|x) - \pi(\de y)) ( f(y) - c) \\
        & \le  \left( \int_{\Xs} \pi(\de x) \left( \frac{|\nu(\de x) - \pi (\de x)|}{\pi(\de x)}\right)^{p}\right)^{\frac{1}{p}} \left( \int_{\Xs} \pi(\de x) \left| \int_{\Xs} (P^t(\de y|x) - \pi(\de y)) (f(y) - c) \right|^q \right)^{\frac{1}{q}} \\
        & = \left\| \frac{\nu}{\pi} - 1 \right\|_{\pi,p} \left\| (P^t - \Pi) (f-c) \right\|_{\pi,q} \\
        & \le \left\| \frac{\nu}{\pi} - 1 \right\|_{\pi,p} \left\| P^t - \Pi  \right\|_{\pi,q \rightarrow q} \left\| f-c \right\|_{\pi,q}
    \end{align*}
    by applying the H\"older's inequality with $p^{-1} + q^{-1} = 1$, and by taking the supremum over $g = f-c$ to obtain the last line. Finally, we take $p = q = 2$ and $c = \pi f$ to prove the result.
\end{proof}

\begin{lemma} \label{lemma:discounted_kernel}
	Let $P: \Xs \rightarrow \PM{\Xs}$ be a Markov kernel, $\nu \in \PM{\Xs}$ be the initial state distribution, and $\gamma \in [0,1]$ be the discount factor. Let $P_\gamma = (1 - \gamma) \mathbf{1} \nu + \gamma P$ the corresponding discounted Markov kernel, it holds that:
	\begin{itemize}
		\item the stationary distribution of $P_{\gamma}$ is the $\gamma$-discounted stationary distribution of $P$, \ie $\pi_{\gamma} = (1-\gamma) \nu + \gamma \pi_\gamma P = (1-\gamma) \nu (I - \gamma P)^{-1}$;
		\item let $\Lambda\left(P^T\right) = \{1, \beta_2, \dots, \beta_{|\Xs|}\}$ with $1 \ge |\beta_2| \ge \dots \ge |\beta_{|\Xs|}|$ be the left spectrum of $P$, then the left spectrum of $P_\gamma$ is given by:
		\begin{align*}
			\Lambda\left(P_\gamma^T\right) = \left\{1, \gamma \beta_2, \dots, \gamma \beta_{|\Xs|} \right\}.
		\end{align*}
	\end{itemize}
\end{lemma}

\begin{proof}
	We start with the first statement. We have to prove that $\pi_\gamma$ is a left eigenvalue of $P_\gamma$:
	\begin{align*}
		\pi_\gamma P_\gamma & = \pi_\gamma \left(  (1-\gamma) \mathbf{1}\nu_0 + \gamma P \right) = (1-\gamma) \underbrace{\pi_\gamma\mathbf{1}}_{=1} \nu + \gamma \pi_\gamma P  \\
		& = (1-\gamma) \nu \underbrace{\left( I  + \gamma   (I - \gamma P)^{-1} P \right) }_{=  (I - \gamma P)^{-1}}
		=  (1-\gamma)\nu (I - \gamma P)^{-1}  = \pi_\gamma.
	\end{align*}
	We move to the second statement. For $\gamma=1$, the statement hold. Thus, we limit to $\gamma \in [0,1)$. From Lemma 1 of~\cite{haveliwala2003second}, we have that $|\beta_2\left(P_\gamma^T\right)| < 1$ for $\gamma \in [0,1)$. Since  $P_\gamma$ is a row stochastic matrix, \ie $P_\gamma \mathbf{1} = \mathbf{1}$, we have that $\mathbf{1}$ is a right eigenvector associated to eigenvalue $1$. From page 4 of~\cite{wilkinson1971algebraic}, we have that left and right eigenvectors associated to different eigenvalues are orthogonal. In particular, we take  $\mathbf{1}$ as right eigenvector associated to eigenvalue $1$ and $x_i$ with $i > 1$ as left eigenvector of $P_\gamma$ associated to eigenvalue $\beta_i$. As $|\beta_i| < 1$, we have that $x_i \mathbf{1} = \mathbf{0}$. Thus, we have:
	\begin{align*}
	x_i P_\gamma = (1-\gamma) x_i \mathbf{1} \nu + \gamma x_i P = \gamma x_i P = \beta_i x_i \implies x_i P = \frac{\lambda_i}{\gamma} x_i. 
	\end{align*}
	Thus, $x_i$ is also eigenvector of $P$ associated to eigenvalue $ \frac{\beta_i}{\gamma}$. We get the result by a change of variable.
\end{proof}

\begin{lemma}\label{lemma:boh}
Let $f(x)= \gamma^x+a\sqrt{x}$ with $\gamma \in [0.3,1]$, $a > 0$, and $x \ge 1$. Then, $\min_{x \ge 1} f(x) \ge \frac{a}{\sqrt{6(1-\gamma)}} $.
\end{lemma}

\begin{proof}
	We perform the following variable substitution:
	\begin{align*}
		y = x \log \frac{1}{\gamma} \quad \implies \quad f(y) = e^{-y} + \frac{a}{\sqrt{\log \frac{1}{\gamma}}} \sqrt{y} = e^{-y} + b \sqrt{y},
	\end{align*}
	with $b = \frac{a}{\sqrt{\log \frac{1}{\gamma}}} $.
	We now vanish the derivative of $f(y)$:
	\begin{align*}
		\frac{\partial f}{\partial y}(y) = -e^{-y} + \frac{b}{2\sqrt{y}} = 0 & \quad \implies \quad -2y e^{-2y} = -\frac{b^2}{2} \quad \\
	& \implies \quad -2y_{0,-1} = W_{0,-1} \left( -\frac{b^2}{2} \right) \\
		& \implies \quad y_{0,-1} =  - \frac{1}{2} W_{0,-1} \left( -\frac{b^2}{2} \right) \quad \\
		& \implies \quad x_{0,-1} = -\frac{1}{2 \log \frac{1}{\gamma}} W_{0,-1} \left( - \frac{a^2}{2\left(\log \frac{1}{\gamma}\right)^2} \right),
	\end{align*}
	where $W_{0,-1}$ denote the two branches of the Lambert function. Clearly, such solutions exist provided that $-b^2/2 \ge -1/e$, \ie for $a \le \sqrt{2 \log (1/\gamma) / e}$. If the solutions exists, we know from the Lambert function that for $-1/e \le z \le 0$, we have $W_{0}(z) \ge W_{-1}(z)$. Thus, $y_{-1} \ge y_{0}$. Furthermore, we have $\frac{\partial f}{\partial y}(y)\rvert_{y \rightarrow 0} = +\infty$ and $\frac{\partial f}{\partial y}(y)\rvert_{y \rightarrow +\infty} = +\infty$. Thus, either both stationary points are inflction points or $y_0$ is a local maximum and $y_{-1}$ is a local minimum. Let us consider the second derivative $\frac{\partial^2 f}{\partial y^2}(y) = e^{-y} - \frac{a}{4y^{3/2}}$, that clearly changes sign at least once for $y \ge 0$. Thus, we have that $y_0$ is a local maximum and $y_{-1}$ is a local minimum. For our purposes, thus, we retain $y_{-1}$ only. To get more usable expressions, we consider the bounds on the Lambert function provided in~\citep[][Theorem 1]{Chatzigeorgiou13}, for $0\le z \le 1/e$:
	\begin{align*}
	-W_{-1}(-z) \begin{cases}
			\le -\log z + \sqrt{2(-1-\log z)} \le - 2 \log z\\
			\ge \frac{1}{3} - \frac{2}{3} \log z + \sqrt{2(-1-\log z)} \ge -\frac{2}{3} \log z
	\end{cases}.
	\end{align*}
	Thus, we have:
	\begin{align*}
		f(y_{-1}) \ge e^{ \log  \frac{b^2}{2}} + b \sqrt{-\frac{1}{3} \log \frac{b^2}{2}} \ge \frac{b}{\sqrt{3}},
		% \sqrt{- \frac{1}{3} \log \frac{b^2}{2} } \ge 0.07 b^3 + b \sqrt{0.3 \log \frac{b^2}{2}}  = 0.07 \frac{a^3}{\left( \log \frac{1}{\gamma} \right)^{3/2}} + \frac{a}{\sqrt{ \log \frac{1}{\gamma} }} \sqrt{0.3 \log \frac{a^2}{2 \left( \log \frac{1}{\gamma} \right)^{2}}}
	\end{align*}
	since $e^{\log \frac{b^2}{2}} = \frac{b^2}{2} \ge 0$ and, since $b^2/3 \le \frac{1}{e}$ it follows that $-\log \frac{b^2}{2} \ge 1$.
	By replacing the value of $b$ defined in terms of $a$ and $\gamma$, we obtain:
	\begin{align*}
		f(x_{-1}) \ge  \frac{a}{\sqrt{3\log \frac{1}{\gamma}}} \ge  \frac{a}{\sqrt{6(1-\gamma)}},
	\end{align*}
	for $\gamma \ge 0.3$.
	
	Instead, in the case, $a < \log (1/\gamma) / \sqrt{e}$, the minimum is attained for $x=0$, \ie $f(0) = 1$.
\end{proof}

%
%\begin{lemma}
%Let $f(x)= \gamma^x+a\sqrt{x}$ with $a> 0$. Then, if $-\frac{a^2}{2\log \frac{1}{\gamma}} \ge -\frac{1}{e}$, $f$ is minimized by taking $x^* = -\frac{W_{-1}\left(-\frac{a^2}{2\log \frac{1}{\gamma}} \right)}{2\log \frac{1}{\gamma}}$, being $W_{-1}$ is the lower branch of the Lambert function. Moreover:
%\begin{align*}
%	f(x^*) \le 
%\end{align*}
%\end{lemma}
%
%\begin{proof}
%	
%\end{proof}
\end{document}